\documentclass{article}

\usepackage{microtype}
\usepackage{graphicx}
\usepackage{booktabs} 
\usepackage{subcaption}

\usepackage{hyperref}

\usepackage[accepted]{icml2024}

\usepackage{amsmath}
\usepackage{amssymb}
\usepackage{mathtools}
\usepackage{amsthm}

\usepackage[capitalize,noabbrev]{cleveref}

\theoremstyle{plain}
\newtheorem{theorem}{Theorem}[section]

\theoremstyle{definition}
\newtheorem{definition}[theorem]{Definition}

\theoremstyle{remark}

\usepackage[textsize=tiny]{todonotes}

\icmltitlerunning{Directly Denosing Diffusion Models}

\begin{document}

\twocolumn[
\icmltitle{Directly Denoising Diffusion Models}



\icmlsetsymbol{equal}{*}

\begin{icmlauthorlist}
\icmlauthor{Dan Zhang}{yyy,equal}, 
\icmlauthor{Jingjing Wang}{yyy,equal}, 
\icmlauthor{Feng Luo}{yyy}

\end{icmlauthorlist}

\icmlaffiliation{yyy}{School of Computing, Clemson University, USA}

\icmlcorrespondingauthor{Feng Luo}{luofeng@clemson.edu}

\icmlkeywords{Machine Learning, ICML}

\vskip 0.3in
]



\printAffiliationsAndNotice{\icmlEqualContribution} 

\begin{abstract}
In this paper, we present the Directly Denoising Diffusion Model (DDDM): a simple and generic approach for generating realistic images with few-step sampling, while multistep sampling is still preserved for better performance. DDDMs require no delicately designed samplers nor distillation on pre-trained distillation models. DDDMs train the diffusion model conditioned on an estimated target that was generated from previous training iterations of its own. To generate images, samples generated from the previous time step are also taken into consideration, guiding the generation process iteratively. We further propose Pseudo-LPIPS, a novel metric loss that is more robust to various values of hyperparameter. Despite its simplicity, the proposed approach can achieve strong performance in benchmark datasets. Our model achieves FID scores of 2.57 and 2.33 on CIFAR-10 in one-step and two-step sampling respectively, surpassing those obtained from GANs and distillation-based models. By extending the sampling to 1000 steps, we further reduce FID score to 1.79, aligning with state-of-the-art methods in the literature. Our code is available at \url{https://github.com/TheLuoFengLab/DDDM}.

\end{abstract}

\section{Introduction}

Diffusion models (DM) recently have attracted significant attention for their exceptional ability to generate realistic samples in recent years, DMs achieved impressive performance in many fields, including image generation \cite{ADM,glide,unclip,Imagen,latentDiff}, video generation \cite{videodiff}, inpainting \cite{latentDiff} and super-resolution \cite{sr3,latentDiff}. 
However, a remarkable drawback is their relatively slow sampling speed, which poses challenges for practical applications. For instance, the vanilla method of DMs (DDPM \citet{ddpm},  Score based models \cite{score}) takes hundreds to thousands of steps in sampling, which is very time-consuming compared with one-step generation such as GAN-style models \cite{bigGAN,stylegan2-ada}, normalizing flow models \cite{glow,flow}, or consistency models \cite{consistency,iCT}. When directly generating the image using DDPM, the accumulated distortion leads to poor performance.

Many efforts to accelerate the sampling process of DM have been proposed. Denoising Diffusion Implicit Models (DDIM, \citet{ddim}) modified the diffusion process into a non-Markovian format with a smaller number of function evaluations (NFE) to generate samples. 
Meanwhile, by viewing the sampling process through the lens of ordinary differential equation (ODE),  \citet{score} developed faster numerical solvers to reduce the NFE required for generation rapidly, thus speeding up the process significantly \cite{3deis,dpm-solver}. 
Although these solvers can achieve comparable results as thousands-of-step samplers, the performance for single-step generation is still not good. Furthermore, knowledge distillation-based methods \cite{knowledge} compress the information learned by the thousands-step sampler into a one-step model, which enables the one-step generation of samples. However, the distillation-based models add computational overhead to the training process as they require another pretrained diffusion model (teacher model) and have potential architectural constraints \cite{iCT}.

In this paper, we propose Directly Denoising Diffusion Model (DDDM) that combines the efficiency of single-step generation with the benefits of iterative sampling for improved sample quality. DDDM employs a DDPM-style noise scheduler and denoises under the probability flow ODE framework. However, we solve the probability flow ODE using a neural network only without using any ODE solver. Our method enables the generation of data samples from random noise with high quality in just one step. Moreover, DDDM can still allow multi-step sampling to obtain better generation results. Furthermore, inspired by Pseudo-Huber losses, we proposed pseudo Learned Perceptual Image Patch Similarity (LPIPS) \cite{lpips}, which shows more robustness in our study. 

 In our experiments, we demonstrate the effectiveness of DDDMs across various image datasets including CIFAR-10 \cite{CIFAR}, and ImageNet 64x64 \cite{imagenet}, and observe comparable results to current state-of-the-art methods. Our model achieves FID scores of 2.57 and 2.33 on CIFAR-10 in one-step and two-step sampling respectively. By extending the sampling to 1000 steps, we further reduce FID score to 1.79. 

Our contributions can be summarized as follows:
\begin{itemize}
    \item We introduce the Directly Denoising Diffusion Models (DDDM) that can obtain the performance for image generation results that is comparable to current state-of-the-art methods and can obtain better generation results for multiple-step sampling.   
    \item Our model provides a straightforward pass with much fewer constraints and does not need ODE solvers.
    \item We proposed the Pseudo-LPIPS metric with increased sensitivity when the loss gets smaller, which is more robust.
\end{itemize}

\section{Preliminary}
\subsection{Diffusion models}
Inspired by non-equilibrium thermodynamics, diffusion models \cite{firstdiff,ddpm,score}, present a generative framework that models data from an unknown true distribution $p_{\text{data}}(\mathbf{x})$. These models consist of two processes: a forward diffusion process and a reverse denoising process.

The forward diffusion process is characterized by the gradual introduction of noise into the original data, denoted as $\mathbf{x}_0$, over a sequence of time steps from 0 to $T$. This process is mathematically structured as a Markov chain, where Gaussian noise is incrementally added to the data at each step. At time step $t$, the distribution of $\mathbf{x}_t$ condition on $\mathbf{x}_{t-1}$ can be expressed as:

\begin{equation*}
    q\left(\mathbf{x}_t \mid \mathbf{x}_{t-1}\right)=\mathcal{N}(\mathbf{x}_t ; \sqrt{1-\beta_t} \mathbf{x}_{t-1}, \beta_t \mathbf{I})
\end{equation*}

By the property of Markov chain, the distribution of $\mathbf{x}_t$ given $\mathbf{x}_0$ is described as \cite{ddpm}:
$$
\left.q\left(\mathbf{x}_t \mid \mathbf{x}_0\right)=\mathcal{N}\left(\mathbf{x}_t ; \sqrt{\bar{\alpha}_t} \mathbf{x}_0,\left(1-\bar{\alpha}_t\right) \mathbf{I}\right)\right)
$$
where $\bar{\alpha}_t=\prod_{s=1}^t\left(1-\beta_s\right)$, $q\left(\mathbf{x}_t \mid \mathbf{x}_0\right)$ is also known as diffusion kernel.

The reverse denoising process aims to learn the inverse of the forward diffusion. Starting from a random sample $\mathbf{x}_T$ with distribution $p\left(\mathbf{x}_T\right)=\mathcal{N}\left(\mathbf{x}_T; \mathbf{0}, \mathbf{I}\right)$, this sampled data is then progressively denoised through a neural network that parameterizes a conditional distribution $q\left(\mathbf{x}_s \mid \mathbf{x}_t\right)$, where $s<t$. This denoising process continues step by step, moving backward in time from step $T$ towards step $0$. The sequence of denoising steps gradually reconstructs the data, aiming to approximate the original data as closely as possible when the time step $0$ is reached.

\subsection{Stochastic Differential Equation Formulation}
The discrete processes in DDPM can be linked to continuous-time diffusion processes \cite{ddim}. By obtaining a continuous approximation of the forward discrete process, we can align it with a Stochastic Differential Equation (SDE) and consequently derive a reverse continuous-time process that corresponds with the reverse discrete process defined in diffusion models.

For the diffusion kernels $\left\{p_{\beta_i}\left(\mathbf{x} \mid \mathbf{x}_0\right)\right\}_{i=1}^N$ used in DDPM, we have:
$$
\mathbf{x}_i=\sqrt{1-\beta_i} \mathbf{x}_{t-1}+\sqrt{\beta_t} \mathbf{z}_{t-1},  \mathbf{z}_{t-1} \sim \mathcal{N}(\mathbf{0}, \mathbf{I}), 
$$
where $t=1, \cdots, T$ and $\beta_t$ can be approximated to an infinitesimal function $\beta(t) \Delta t$ as $ T \rightarrow \infty,$ and $\beta_t$ sufficiently small. 

Applying Taylor expansion, the following can be derived:
$$\mathbf{x}_t \approx \mathbf{x}_{t-1}-\frac{\beta(t) \Delta t}{2} \mathbf{x}_{t-1}+\sqrt{\beta(t) \Delta t} \mathbf{z}_{i-1}.$$

As the time increment $\Delta t \rightarrow 0$, the above discrete function transitions into the following Variance Preserving (VP) Stochastic Differential Equation (SDE) \cite{ddim}:

$$
\mathrm{d} \mathbf{x}_t=-\frac{1}{2} \beta(t) \mathbf{x}_t \mathrm{d} t+\sqrt{\beta(t)} \mathrm{d} \mathbf{w},
$$
where $\mathbf{w}$ is a Wiener process.

The reverse process for this VP SDE is defined as:
$$
\mathrm{d} \mathbf{x}_t=\left(-\frac{1}{2} \beta(t) \mathbf{x}_t -\beta(t) \nabla_x \log q(\mathbf{x}_t)\right) \mathrm{d} t+\sqrt{\beta(t)} \mathrm{d} \mathbf{w}
$$

The Generative probability flow Ordinary Differential Equation (ODE), which is deterministic, can be expressed as:
\begin{equation}
    \frac{\mathrm{d} \mathbf{x}_t}{\mathrm{~d} t}=-\frac{1}{2} \beta(t)\left[\mathbf{x}_t-\nabla_{x_t} \log q_t\left(\mathbf{x}_t\right)\right] .
    \label{equ:ode}
\end{equation}

By replacing $\nabla_{\boldsymbol{x}} \log q_t\left(\boldsymbol{x}_t\right)$ with a neural network based estimation $\mathbf{s}_{\boldsymbol{\theta}}\left(\mathbf{x}_t, t\right)$, \citet{score}  obtained the neural ODE:
$$\mathrm{d} \mathbf{x}_t=-\frac{1}{2} \beta(t)\left[\mathbf{x}_t+\mathbf{s}_{\boldsymbol{\theta}}\left(\mathbf{x}_t, t\right)\right] \mathrm{d} t $$

Advanced ODE solvers can be applied to solve the above equations. 


\begin{figure*}
    \centering
    \includegraphics[scale=0.5]{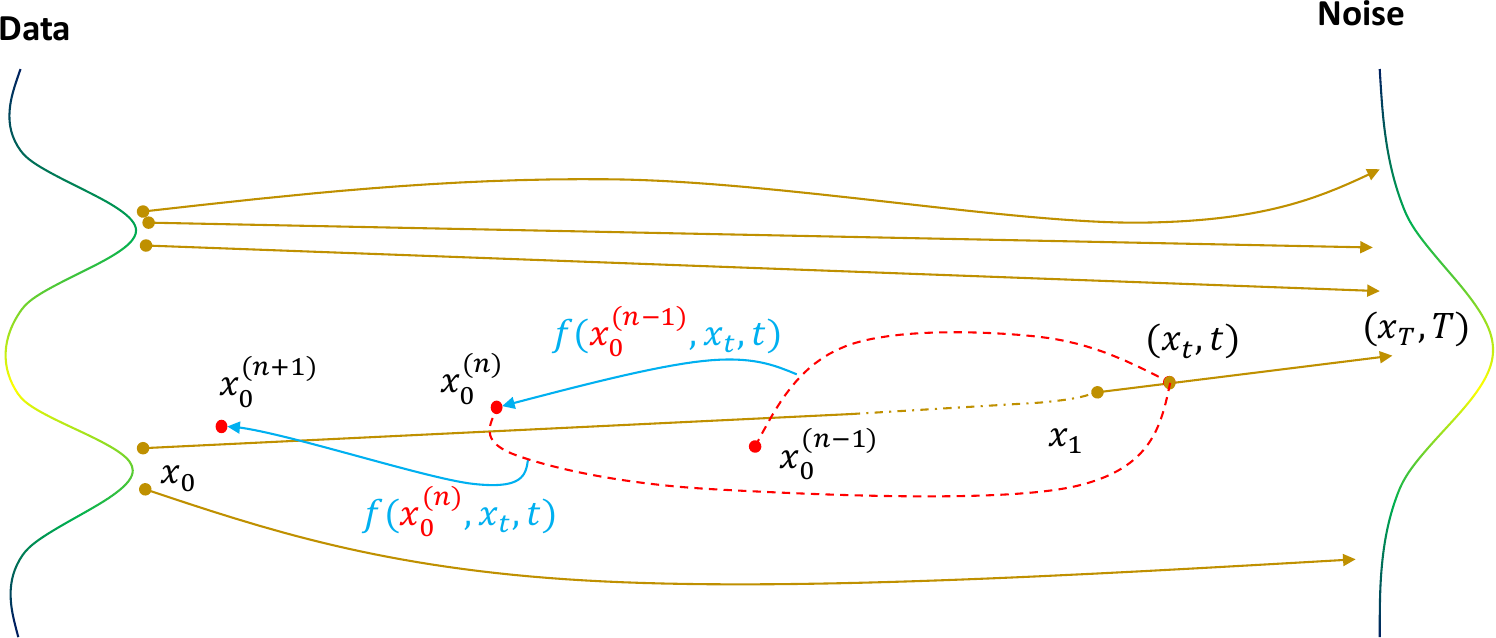}
    \caption{An illustration of DDDM. For current training epoch $n$, our model takes noisy data $\mathbf{x}_t$ and timestep $t$, as well as the estimated target from previous epoch $\mathbf{x}_0^{(n-1)}$ as inputs, predicts the new approximation $\mathbf{x}_0^{(n)}$, which will be utilized in the next training epoch. Through such an iterative process, our approximated result moves gradually towards real data $\mathbf{x}_0$.}
    \label{fig:arch}
\end{figure*}

\section{Directly Denosing Diffusion Models}


Solving the probability flow (PF) ODE is equivalent to computing the following integral: 

$\int_T^0 \frac{\mathrm{d} \mathbf{x}_t}{\mathrm{~d} t} \mathrm{~d} t=\int_T^0 -\frac{1}{2} \beta(t)\left[\mathbf{x}_t-\nabla_{x_t} \log q_t\left(\mathbf{x}_t\right)\right] \mathrm{~d} t \Longleftrightarrow \mathbf{x}_0=\mathbf{x}_T+\int_T^0 -\frac{1}{2} \beta(t)\left[\mathbf{x}_t-\nabla_{x_t} \log q_t\left(\mathbf{x}_t\right)\right]  \mathrm{~d} t$

where $\mathbf{x}_T$ is initialized from a normal distribution  $\mathcal{N}(\mathbf{0}, \mathbf{I})$.

To generate samples from a DM, we propose Directly Denosing Diffusion Models, an iterative process designed to refine the estimation of $\mathbf{x}_0$.
First, we define $\mathbf{f}\left(\mathbf{x}_0, \mathbf{x}_t, t\right)$
as the solution of the PF ODE from initial time $t$ to final time $0$ (Appendix \ref{appen:1}):
$$\mathbf{f}\left(\mathbf{x}_0, \mathbf{x}_t, t\right) := \mathbf{x}_t+\int_t^0-\frac{1}{2} \beta(s) \left[\mathbf{x}_s- \nabla_{x_s} \log q_s\left(\mathbf{x}_s\right)\right] \mathrm{d} s $$
 where $\mathbf{x}_t$ is drawn from  $\mathcal{N}\left(\sqrt{\bar{\alpha}_t} \mathbf{x}_0,\left(1-\bar{\alpha}_t\right) \mathbf{I}\right)$. Subsequently, we introduce the function $\mathbf{F}\left(\mathbf{x}_0, \mathbf{x}_t, t\right)$ defined as:
$$\mathbf{F}\left(\mathbf{x}_0, \mathbf{x}_t, t\right) :=\int_t^0\frac{1}{2} \beta(s) \left[\mathbf{x}_s- \nabla_{x_s} \log q_s\left(\mathbf{x}_s\right)\right] \mathrm{d} s$$
Thus, we have:
\begin{equation}
\mathbf{f}\left(\mathbf{x}_0, \mathbf{x}_t, t\right) = \mathbf{x}_t - \mathbf{F}\left(\mathbf{x}_0, \mathbf{x}_t, t\right)
 \label{equ:1}
\end{equation}
By approximate $\mathbf{f}$, we can then recover the original image $\mathbf{x}_0$. We define a neural network-parameterized function $\mathbf{f}_{\boldsymbol{\theta}}$, which is employed to estimate the solution of the PF ODE and thereby recover the original image state at time 0. The predictive model is represented as:
\begin{equation}
   \mathbf{f}_{\boldsymbol{\theta}}\left(\mathbf{x}_0, \mathbf{x}_t, t\right)=\mathbf{x}_t-\mathbf{F}_{\boldsymbol{\theta}}\left(\mathbf{x}_0, \mathbf{x}_t, t\right)
 \label{equ:1}
\end{equation}

where $\mathbf{F}_{\boldsymbol{\theta}}$ is the neural network function parameterized by the weights $\boldsymbol{\theta}$. To achieve a good recovery of the initial state $\mathbf{x}_0$, $\mathbf{f}_{\boldsymbol{\theta}}\left(\mathbf{x}_0, \mathbf{x}_t, t\right)\approx \mathbf{f}\left(\mathbf{x}_0, \mathbf{x}_t, t\right)$ need to be ensured.

\subsection{Iterative solution} Eq. \eqref{equ:1} shows that our neural network $\mathbf{F}_{\boldsymbol{\theta}}$ requires $\mathbf{x}_0$ as input, which is not applicable during sample generation. To unify the training and inference within the same framework, we propose an iterative update rule to estimate the initial state $\mathbf{x}_0$ in a dynamic system. This iterative process is formally defined by the following update equation:
\begin{equation}
\mathbf{x}_0^{(n+1)}=\mathbf{x}_t-\mathbf{F}_{\boldsymbol{\theta}}\left(\mathbf{x}_0^{(n)}, \mathbf{x}_t, t\right) 
\end{equation}  
where $\mathbf{x}_0^{(n)}$ denotes the estimated ground truth data $\mathbf{x}_0$ at the $n$-th training epoch or $n$-th sampling iteration. Each update refines this estimate in an attempt to converge to the true initial state. To effectively quantify the discrepancy between the $n$-th estimate $\mathbf{x}_0^{(n)}$ and the true initial state $\mathbf{x}_0$ in the DDDM, we employ the following loss function. 






\begin{definition}
The loss function of the DDDM at the $n$-th iteration is defined as:
\begin{equation}
 \label{eq:loss}
  \begin{aligned}
\mathcal{L}_{\text{DDDM}}^{(n)}(\boldsymbol{\theta}):= &  \mathbb{E}_{t \sim \mathcal{U}[1,T]}[\mathbb{E}_{\mathbf{x}_0 \sim p_{\text {data }}(\mathbf{x}_0)}\\& 
[\mathbb{E}_{\mathbf{x}_t \sim \mathcal{N}(\sqrt{\bar{\alpha}_t} \mathbf{x}_0,(1-\bar{\alpha}_t) \mathbf{I})} \\& 
[d (\mathbf{f}_{\boldsymbol{\theta}}(\mathbf{x}_0^{(n)}, \mathbf{x}_t, t), \mathbf{x}_0)]]]
\end{aligned}  
\end{equation}

where $\mathcal{U}\left[1, T\right]$ denotes a uniform distribution over the integer set $\left[1,2, \cdots, T\right]$. $d(\cdot, \cdot)$ is a metric function satisfies that for all vectors $\mathbf{x}$ and $\mathbf{y}$, $d(\mathbf{x}, \mathbf{y}) \geq 0$ and $d(\mathbf{x}, \mathbf{y})=0$ if and only if $\mathbf{x}=\mathbf{y}$. Therefore, commonly used metrics such as $L_1$, $L_2$ can be utilized. We will discuss our choice of $d(\cdot, \cdot)$ in Section \ref{sec:loss}.

\end{definition}
This definition encapsulates the expected discrepancy between the estimated state $\mathbf{x}_0^{(n)}$ and the true initial state $\mathbf{x}_0$, integrated over a probabilistic model of the data and the time domain.

\textbf{Training}. Each data sample $\mathbf{x}_0$ is chosen randomly from the dataset, following the probability distribution $p_{\text {data }}(\mathbf{x}_0)$. This initial data point forms the basis for generating a trajectory. Next, we randomly sample a timestep $t \sim \mathcal{U}[1, T]$, and obtain 
its noisy variant $\mathbf{x}_t$ from distribution $\mathcal{N}\left(\sqrt{\bar{\alpha}_t} \mathbf{x}_0,\left(1-\bar{\alpha}_t\right) \mathbf{I}\right)$. we play the reparameterization trick to rewrite $\mathbf{x}_t = \sqrt{\bar{\alpha}_t} \mathbf{x}_0 + \sqrt{1-\bar{\alpha}_t}\boldsymbol{\epsilon}, \boldsymbol{\epsilon} \sim \mathcal{N}(\mathbf{0}, \mathbf{I})$. For current training epoch $n$, our model takes noisy data $\textbf{x}_t$ and timestep $t$, as well as the corresponding estimated target from previous epoch $\textbf{x}_0^{(n-1)}$ as inputs, predicts a new approximation $\textbf{x}_0^{(n)}$, which will be utilized in the next training epoch for the same target sample. DDDM is trained by minimizing the loss following Eq. \ref{eq:loss}.
 The full procedure of training DDDM is summarized in Algorithm \ref{algo:1}.

\textbf{Sampling}. The generation of samples is facilitated through the use of a well-trained DDDM, denoted as $\boldsymbol{f}_{\boldsymbol{\theta}}(\cdot, \cdot)$. The process begins by drawing from the initial Gaussian distribution, where both $\mathbf{x}_0^{(0)}$ and $\mathbf{x}_T$ are sampled from $\mathcal{N}\left(\mathbf{0}, \boldsymbol{I}\right)$. Subsequently, these noise vectors and embedding of $T$ are passed through the DDDM model to obtain $\mathbf{x}_0^{\text{est}}=
\mathbf{f}_{\boldsymbol{\theta}}\left(\mathbf{x}_0^{(0)}, \mathbf{x}_T, T\right)
$. This approach is noteworthy for its efficiency, as it requires only a single forward pass through the model. Our model also supports a multistep sampling procedure for enhanced sample quality. Detail can be found in Algorithm \ref{algo:2}.

\begin{algorithm}[t]

   \caption{Training}
\begin{algorithmic}
   \STATE {\bfseries Input:} image dataset $D,$ $T$, model parameter $\boldsymbol\theta$ \\
    initialize $\mathbf{x}_0^{(0)} \sim \mathcal{N}(\mathbf{0}, \mathbf{I})$,
    epoch $n \leftarrow 0$
   \REPEAT
   \STATE Sample $\mathbf{x}_0 \sim D$ and $t \sim \mathcal{U}\left[1, T\right]$.
   \STATE Sample $\boldsymbol{\epsilon} \sim \mathcal{N}(\mathbf{0}, \mathbf{I})$
   \STATE $\mathbf{x}_t = \sqrt{\bar{\alpha}_t} \mathbf{x}_0+\sqrt{1-{\bar\alpha}_t} \boldsymbol{\epsilon} $
   \STATE$\mathbf{x}_0^{(n+1)} \leftarrow \mathbf{x}_t-\mathbf{F}_{\boldsymbol\theta}\left(\mathbf{x}_{0}^{(n)}, \mathbf{x}_t, t\right)$ 
   \STATE $\mathcal{L}_{\text{DDDM}}^{(n+1)}(\boldsymbol{\theta})\leftarrow d\left(\boldsymbol{f}_{\boldsymbol{\theta}}\left(\mathbf{x}_0^{(n)}, \mathbf{x}_t, t\right), \mathbf{x}_0\right)$
   \STATE $\boldsymbol{\theta} \leftarrow \boldsymbol{\theta}-\eta \nabla_{\boldsymbol{\theta}} \mathcal{L}\left(\boldsymbol{\theta}\right)$
   \STATE $n \leftarrow n+1$

   \UNTIL{convergence}
\end{algorithmic}
\label{algo:1}
\end{algorithm}


Here, we provide theoretical justifications for the convergence of our method below. 

\begin{algorithm}[t]
   \caption{Sampling}
\begin{algorithmic}
   \STATE {\bfseries Input:} $T$, trained model parameter $\boldsymbol\theta$, sampling step $s$
   $\mathbf{x}_0^{(0)} \sim \mathcal{N}(\mathbf{0}, \mathbf{I})$, $\mathbf{x}_T \sim \mathcal{N}(\mathbf{0}, \mathbf{I})$ 
   \FOR{$n=0$ {\bfseries to} $s-1$}
   \STATE $\mathbf{x}_0^{(n+1)} \leftarrow \mathbf{x}_T-\mathbf{F}_{\boldsymbol\theta}\left(\mathbf{x}_{0}^{(n)}, \mathbf{x}_T, T\right)$ 
   \ENDFOR
   \STATE \textbf{Output:} $\mathbf{x}_0^{(n+1)}$
   
\end{algorithmic}
\label{algo:2}
\end{algorithm}

\begin{theorem}\label{th1}
If the loss function 
$\mathcal{L}_{\text{DDDM}}^{(n)}\left(\boldsymbol{\theta}\right)\rightarrow 0 $ as $n\rightarrow \infty$, it can be shown that as $n\rightarrow \infty$,

\begin{equation}\label{eq1}
 \sup_{\mathbf{x}_{0}}\left(\mathbf{f}_{\boldsymbol{\theta}}\left(\mathbf{x}_{0}^{(n)},\mathbf{x}_{t}, t\right)- \mathbf{f}\left(\mathbf{x}_{0},\mathbf{x}_{t}, t\right)\right)\rightarrow 0   
\end{equation}

\end{theorem}

\begin{proof}
As $n\rightarrow \infty, \mathcal{L}_{\text{DDDM}}^{(n)}\left(\boldsymbol{\theta}\right)\rightarrow 0 $, we have
$$
 \mathbb{E}\left[d\left(\mathbf{f}_{\boldsymbol{\theta}}\left(\mathbf{x}_{0}^{(n)},\mathbf{x}_{t}, t\right), \mathbf{x}_0\right)\right]\rightarrow 0
$$

According to the definition, we have 

$p\left(\mathbf{x}_{t}\right)> 0$ for every $\mathbf{x}_{t}$ and $1 \leqslant t \leqslant T$. Therefore, we have:
$$
d\left(\mathbf{f}_{\boldsymbol{\theta}}\left(\mathbf{x}_{0}^{(n)},\mathbf{x}_{t}, t\right), \mathbf{x}_0\right)\rightarrow 0 
$$

Because $d(\mathbf{x}, \mathbf{y})=0  \text{ if and only if } \mathbf{x}=\mathbf{y}$, this indicates that:

for any $\mathbf{x}_0$ sampled from $p_{\text {data}}(\mathbf{x}_0)$,
$$
\mathbf{f}_{\boldsymbol{\theta}}\left(\mathbf{x}_0^{(n)}, \mathbf{x}_t, t\right) \rightarrow \mathbf{x}_0 \text{, as }  n \rightarrow \infty,
$$
which implies:
$$\sup _{\mathbf{x}_0} d\left(\mathbf{f}_{\boldsymbol{\theta}}\left(\mathbf{x}_0^{(n)}, \mathbf{x}_t, t\right)-\mathbf{f}\left(\mathbf{x}_0, \mathbf{x}_t, t\right)\right) \rightarrow 0 .$$
\end{proof}

\begin{figure*}[t!]
 \centering
\begin{subfigure}[t]{0.48\textwidth}
  \includegraphics[scale=0.5]{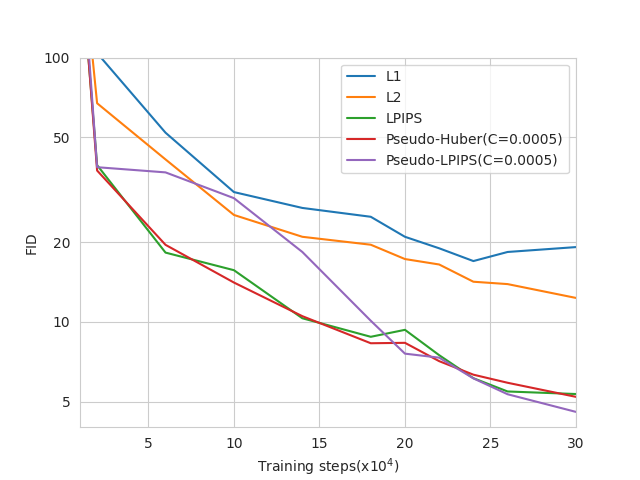}  
  \caption{}
  \label{fig:fid_all}
\end{subfigure}
\begin{subfigure}[t]{0.48\textwidth}
  \includegraphics[scale=0.5]{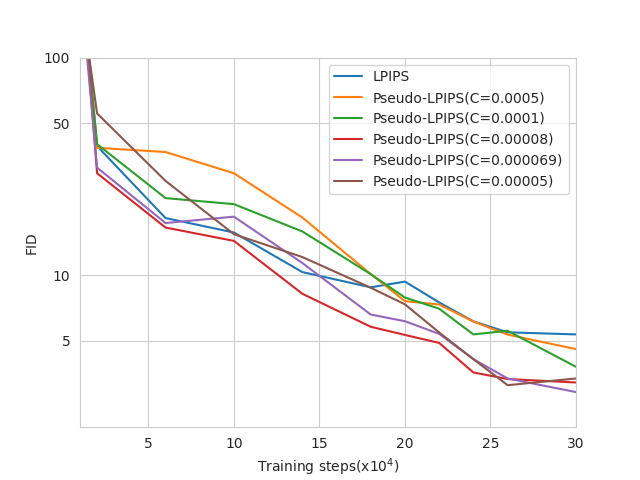}  
  \caption{}
  \label{fig:fid_c}
\end{subfigure}
\caption{ Ablation analysis for our proposed Pseudo-LPIPS metric. (a) While LPIPS and Pseudo-Huber perform closely, Pseudo-LPIPS further reduces FID to under 5. (b) Pseudo-LPIPS outperforms LPIPS with various values of hyperparameter $c$, where $c=0.000069$ is the best. The y-axis for both figures is scaled logarithmically for better visualization.}
\label{fig:fid}
\end{figure*}

The following theorem draws inspiration from Proposition 3 in \citet{kim2023consistency}, leading us to a similar conclusion.
\begin{theorem}
Suppose the following conditions are met:
 
(i) For all $\mathbf{x}, \mathbf{y} \in \mathbb{R}^D$, time $t \in\mathcal{U}[1, T]$, $\mathbf{x}_t\sim \mathcal{N}\left(\sqrt{\bar{\alpha}_t} \mathbf{x}_0,\left(1-\bar{\alpha}_t\right) \mathbf{I}\right)$, $\mathbf{y}_t \sim \mathcal{N}\left(\sqrt{\bar{\alpha}_t} \mathbf{y}_0,\left(1-\bar{\alpha}_t\right) \mathbf{I}\right)$, the function $\mathbf{f}_{\boldsymbol{\theta}}$ satisfies the uniform Lipschitz condition:
$$
\sup _{\boldsymbol{\theta}}\left\|\mathbf{f}_{\boldsymbol{\theta}}\left(\mathbf{x}, \mathbf{x}_t, t\right)-\mathbf{f}_{\boldsymbol{\theta}}\left(\mathbf{y}, \mathbf{y}_t, t\right)\right\|_2 \leq L\|\mathbf{x}-\mathbf{y}\|_2
$$
where $L$ is a Lipschitz constant.

(ii) There exists a non-negative function $L(\mathbf{x})$  such that for all $\mathbf{x} \in \mathbb{R}^D$, time $t \in\mathcal{U}[1, T]$ and $\mathbf{x}_t \sim \mathcal{N}\left( \sqrt{\bar{\alpha}_t} \mathbf{x},\left(1-\bar{\alpha}_t\right) \mathbf{I}\right)$, the function $\mathbf{f}_{\boldsymbol{\theta}}$ is uniformly bounded in $\boldsymbol{\theta}$ :
$$
\sup _{\boldsymbol{\theta}}\left\|\mathbf{f}_{\boldsymbol{\theta}}\left(\mathbf{x}, \mathbf{x}_t, t\right)\right\|_2 \leq L(\mathbf{x})<\infty \text {. }
$$ 
If there exists a $\boldsymbol{\Tilde{{\theta}}}$ such that the loss function 
$\mathcal{L}_{\text{DDDM}}\left(\boldsymbol{\Tilde{{\theta}}}\right)\rightarrow 0$ as the iteration number $n$ becomes sufficiently large. Let $p_{\boldsymbol{\Tilde{{\theta}}}}(\cdot)$ denote the pushforward distribution of $p_T$ induced by $\mathbf{f}_{\boldsymbol{\Tilde{{\theta}}}}(\cdot, \mathbf{x}_T, T)$.

Then, under these conditions, the discrepancy between pushforward distribution $p_{\boldsymbol{\Tilde{{\theta}}}}(\cdot)$ and the data distribution $p_{\text {data }}(\cdot)$
measured in the uniform norm, converges to $0$ as the number of iterations $n$ approaches infinity:
$\left\|p_{\boldsymbol{\Tilde{{\theta}}}}(\cdot)-p_{\text {data }}(\cdot)\right\|_{\infty} \rightarrow 0$ as iteration $n \rightarrow \infty$.

\end{theorem}

\begin{proof}
Based on Theorem 3.2, let $t = T$ we have that for sufficiently large $n$,
$$\sup _{\mathbf{x}_0}\left\|\mathbf{f}_{\Tilde{{\theta}}}\left(\mathbf{x}_0^{(n)}, \mathbf{x}_T, T\right)-\mathbf{f}\left(\mathbf{x}_0, \mathbf{x}_T, T\right)\right\|_2 \rightarrow 0
$$

which implies that 
$$\mathbf{f}_{\tilde{\theta}}\left(\mathbf{x}_0^{(n)}, \mathbf{x}_T, T\right)\rightarrow\mathbf{f}\left(\mathbf{x}_0, \mathbf{x}_T, T\right)$$
when $n$ is large enough. Then we conclude that the pushforward distribution of $\mathbf{x}_T$, say $p_{\boldsymbol{\Tilde{\theta}}}(\cdot)$, converges in distribution to the data distribution $p_{\text{data}(\cdot)}$. Since for all $\mathbf{x}, \mathbf{y} \in \mathbb{R}^D, t \sim \mathcal{U}[1, T]$, and $\boldsymbol{\theta}$,$\left\{\mathbf{f}_{\boldsymbol{\theta}}\right\}_{\boldsymbol{\theta}}$ is uniform Lipschitz
$\left\|\mathbf{f}_{\boldsymbol{\theta}}\left(\mathbf{x}, \mathbf{x}_t, t\right)-\mathbf{f}_{\boldsymbol{\theta}}\left(\mathbf{y}, \mathbf{y}_t, t\right)\right\|_2\leq L\left\|\mathbf{x}-\mathbf{y}\right\|_2, \quad$ $\left\{\mathbf{f}_{\boldsymbol{\theta}}\right\}_{\boldsymbol{\theta}}$ is asymptotically uniformly equicontinuous (a.u.e.c.). Additionally, $\left\{\mathbf{f}_{\boldsymbol{\theta}}\right\}_{\boldsymbol{\theta}}$ is uniform bounded in $\boldsymbol{\Tilde{\theta}}$. Thus, by the converse of Scheffé's theorem \cite{boos1985converse,sweeting1986converse}, it suggests $\left\|p_{\boldsymbol{\Tilde{\theta}}}(\cdot)-p_{\text {data }}(\cdot)\right\|_{\infty} \rightarrow 0$ as $n$ sufficiently large.
\end{proof}

\section{Psuedo LPIPS Metric}

\label{sec:loss}
Assessment of image quality becomes increasingly crucial in image generation. The Learned Perceptual Image Patch Similarity (LPIPS, \citet{lpips}) metric has been a significant tool to help improve the quality of generated images. However, the LPIPS is still not robust enough to outliers in practice. Inspired by \citet{iCT} recent using the pseudo-Huber loss to significantly improve the robustness of the training consistency model, we propose a modified version of the LPIPS metric, defined as Pseudo-LPIPS:
\begin{equation}
    \text{Pseudo-LPIPS} = \sqrt{\text{LPIPS} + c^2} - c
\end{equation}
where $c$ is an adjustable hyperparamter. Similar to the Pseudo-Huber metric, the inclusion of the term $c^2$ and the subsequent square root transformation in the Pseudo-LPIPS metric aim to provide a more balanced and perceptually consistent measure. This approach mitigates the overemphasis on larger errors and increases the sensitivity and accuracy of the metric in discerning perceptual differences in images.

The merits of the Pseudo-LPIPS Metric are as follows:
\begin{itemize}
\item Enhanced Sensitivity to Perceptual Differences: The modified metric is finely attuned to subtle perceptual variances, often missed by traditional metrics. This sensitivity is especially valuable in fields requiring high-precision image quality, like medical imaging or high-fidelity rendering.
\item Balanced Error Emphasis: It provides a more equitable emphasis across various error magnitudes, contrasting the $L_2$ norm's tendency to disproportionately penalize larger errors.
\item Adaptability: The incorporation of the constant $c$ allows for flexibility, making the metric versatile for different scenarios and datasets.
\item Improved Robustness: The metric is more resilient against outliers and anomalies due to the square root transformation, addressing a common flaw in the Pseudo-Huber loss.
\end{itemize}
When comparing this modified metric to the $L2$ norm and pseudo-Huber loss, the Modified LPIPS aligns more closely with human perceptual judgments. The $L2$ norm, while simple in its mathematical form, often fails to accurately represent human vision. The Pseudo-Huber loss, although attempting to merge the benefits of $L1$ and $L2$ norms, sometimes falls short in providing a balanced representation of perceptual quality. The Pseudo-LPIPS, through its nuanced formulation, effectively bridges these gaps, presenting a metric that is both perceptually meaningful and mathematically sound.

\begin{table}[t!]
\setlength\tabcolsep{0.01cm}
    \centering
    \caption{Comparing the quality of unconditional samples on CIFAR-10}
    \fontsize{8.5pt}{8.5pt}\selectfont
    \begin{tabular}{lccc}
    \toprule
         Method & NFE($\downarrow$) & FID($\downarrow$) & IS($\uparrow$) \\
         \\
         \multicolumn{4}{l}{\textbf{Fast samplers \& distillation for diffusion models}} \\
         \toprule
         DDIM \cite{ddim}& 10 & 13.36 & \\
         DPM-solver-fast \cite{dpm-solver}& 10 & 4.70 \\
         3-DEIS \cite{3deis}& 10 & 4.17 \\
         UniPC \cite{unipc}& 10 & 3.87 \\
         DFNO (LPIPS) \cite{DFNO} &1 &3.78 \\
         2-Rectified Flow \cite{rectified} &1 &4.85 &9.01\\
         Knowledge Distillation \cite{knowledge} & 1 & 9.36\\
         TRACT \cite{tract}&1 &3.78 \\
         &2 &3.32 \\
         Diff-Instruct \cite{instruct}&1 &4.53 &9.89 \\
         CD (LPIPS) \cite{consistency} &1 &3.55 &9.48 \\
         &2 &2.93 &9.75 \\
         
         \textbf{Direct Generation} \\
         \toprule
         Score SDE \cite{score}&2000 &2.38 &9.83 \\
         Score SDE (deep) \cite{score}  &2000 &2.20 &9.89 \\
DDPM \cite{ddpm} &1000 &3.17 &9.46 \\
LSGM  \cite{LSGM}&147 &2.10 \\
PFGM  \cite{pfgm} &110 &2.35 &9.68 \\
EDM   \cite{edm}&35  &2.04 &9.84 \\
EDM-G++ \cite{edmG}&35 &1.77 \\
NVAE  \cite{nvae}&1 &23.5 &7.18 \\
BigGAN  \cite{bigGAN}&1 &14.7 &9.22 \\
StyleGAN2 \cite{style2}&1 &8.32 &9.21 \\
StyleGAN2-ADA \cite{stylegan2-ada}  &1 &2.92 &9.83 \\
CT (LPIPS) \cite{consistency}&1 &8.70 &8.49 \\
&2 &5.83 &8.85 \\
iCT \cite{iCT} &1 &2.83 &9.54 \\
&2 &2.46 &9.80 \\
iCT-deep \cite{iCT} &1 & 2.51 & 9.76 \\
&2 &2.24 &9.89 \\
\textbf{DDDM(T=1000)} &1 &2.90 &9.81 \\
&2 &2.79 &9.89 \\
&1000 &1.87 &9.94 \\

\textbf{DDDM(T=8000)} &1 &2.82 &9.83 \\
&2 &2.53 &9.84 \\
&1000 &1.74 &9.93 \\

\textbf{DDDM-deep(T=1000)} &1 &2.57 &9.91 \\
&2 &2.33 &9.91 \\
&1000 &1.79&9.95\\
\bottomrule
         
    \end{tabular}
    
    \label{tab:cifar}
\end{table}

\begin{table}[t!]
\setlength\tabcolsep{0.01cm}
    \centering
    \caption{Comparing the quality of class-conditional samples on ImageNet 64x64}
    \fontsize{8.5pt}{8.5pt}\selectfont
    \begin{tabular}{lcccc}
    \toprule
         Method & NFE($\downarrow$) & FID($\downarrow$) & Prec.($\uparrow$) & Rec.($\uparrow$)\\
         \\
         \multicolumn{5}{l}{\textbf{Fast samplers \& distillation for diffusion models}} \\
         \toprule
        DDIM \cite{ddim}&50 &13.7 &0.65 &0.56 \\
        &10 &18.3 &0.60 &0.49 \\
    DPM solver \cite{dpm-solver} &10 &7.93 \\
    &20 &3.42 \\
DEIS \cite{3deis}&10 &6.65\\
&20 &3.10 \\
DFNO (LPIPS) \cite{DFNO} &1 &7.83 &0.61 \\
TRACT \cite{tract}&1 &7.43 \\
&2 &4.97 \\
BOOT \cite{boot}&1 &16.3 &0.68 &0.36 \\
Diff-Instruct \cite{instruct} &1 &5.57\\

PD (LPIPS) \cite{consistency}  &1 &7.88 &0.66 &0.63 \\
&2 &5.74 &0.67 &0.65 \\
&4 &4.92 &0.68 &0.65 \\
CD (LPIPS) \cite{consistency}  &1 &6.20 &0.68 &0.63 \\
&2 &4.70 &0.69 &0.64 \\
&3 &4.32 &0.70 &0.64 \\
         
         \textbf{Direct Generation} \\
        \toprule
RIN \cite{rin} &1000 &1.23 \\
DDPM  \cite{ddpm}&250 &11.0 &0.67 &0.58 \\
iDDPM \cite{iddpm}&250 &2.92 &0.74 &0.62\\
ADM  \cite{ADM}&250 &2.07 &0.74 &0.63 \\
EDM  \cite{edm}&511 &1.36 \\
BigGAN-deep \cite{bigGAN} &1 &4.06 &0.79 &0.48 \\
CT (LPIPS) \cite{consistency} &1 &13.0 &0.71 &0.47 \\
&2 &11.1 &0.69 &0.56 \\
iCT \cite{iCT}  &1 &4.02 &0.70 &0.63 \\
&2 &3.20 &0.73 &0.63 \\
iCT-deep \cite{iCT}  &1 &3.25 &0.72 &0.63 \\
&2 &2.77 &0.74 &0.62 \\

\textbf{DDDM(T=1000)} &1&4.21&0.71&0.64\\
&2&3.53&0.73&0.64\\
&1000 &2.76 &0.75 &0.65 \\

\textbf{DDDM-deep(T=1000) } &1&3.47&0.71&0.63\\
&2&3.08&0.74&0.66 \\
&1000 &2.11&0.73&0.67 \\
\bottomrule
    \end{tabular}
    
    \label{tab:imagenet}
\end{table}

\section{Experiments}
To evaluate our method for image generation, we train several DDDMs on CIFAR-10 \cite{CIFAR} and ImageNet 64x64 \cite{imagenet} and benchmark their performance with
competing methods in the literature. Results are
compared according to Frechet Inception Distance (FID, \citet{FID}), which is computed between 50K generated samples and the whole training set. We also employ Inception Score (IS,
\citet{IS}) and Precision/Recall \cite{Precision} to measure sample quality.

\begin{figure*}[t!]
 \centering
\begin{subfigure}[t]{0.47\textwidth}
  
  \includegraphics[scale=0.45]{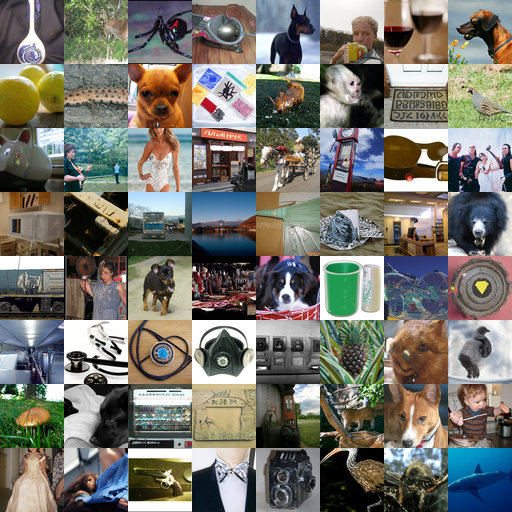}  
  \caption{one-step}
  \label{fig:image_1-step}
\end{subfigure}
\hspace*{0.2cm}
\begin{subfigure}[t]{0.47\textwidth}
  
  \includegraphics[scale=0.45]{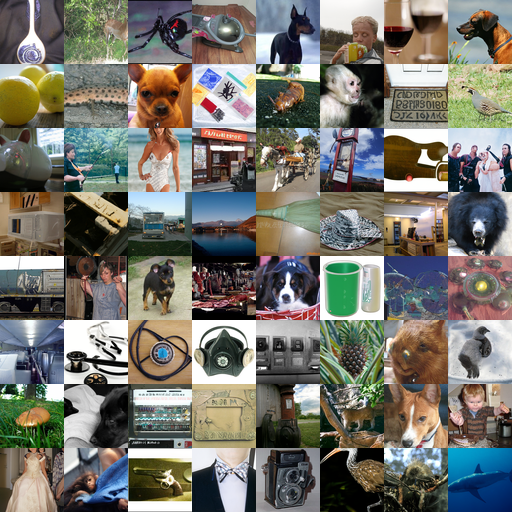}  
  \caption{two-step}
  \label{fig:image_2-step}
\end{subfigure}
\caption{One-step and two-step samples from DDDM-deep model trained on ImageNet 64x64 }
\label{fig:ImageNet}
\end{figure*}

  
  

\subsection{Implementation Details}
\textbf{Architecture.} We use the U-Net architecture from ADM \cite{ADM} for both datasets. For CIFAR-10, we use a base channel dimension of 128, multiplied
by 1,2,2,2 in 4 stages and 3 residual blocks per stage. Dropout \cite{dropout} of 0.3 is utilized for this task. For ImageNet 64x64, we use a base channel dimension of 192, multiplied by 1,2,3,4 in 4 stages and 3 residual blocks per stage, which account for a total of 270M parameters. Following ADM, we employ cross-attention modules not only at the 16x16 resolution but also at the 8x8 resolution, through which we incorporate the conditioning image $\mathbf{x}_0^{(n)}$ into the network. We also explore deeper variants of these architectures by doubling the number of blocks at each resolution, which we name DDDM-deep. All models on CIFAR-10 are unconditional, and all models on ImageNet 64x64 are conditioned on class labels. 

\textbf{Other settings.} We use Adam for all of our experiments.
For CIFAR-10, we set $T=1000$ for baseline model and train the model for 1000 epochs with a constant learning rate of 0.0002 and batch size of 1024. We also explore models with larger $T$ values and longer training epochs. Details can be found in \cref{tab:higher}. For
ImageNet 64×64, we only investigate $T=1000$ due to time constraints and train the model for 520 epochs with a constant learning rate of 0.0001 and batch size of 1024. We use an exponential moving average (EMA) of the weights during training with a decay factor of 0.9999 for all the experiments. All models are trained on 8 Nvidia A100 GPUs.

\subsection{Ablations}
In this section, we ablate various metrics employed in the loss function. We evaluate the effectiveness of the proposed Pseudo-LPIPS metric by training several models with varying $c$ values and comparing the sample qualities with models trained with ${L}_1$, squared ${L}_2$, LPIPS, and Pseudo-Huber on CIFAR-10. As depicted by Figure \ref{fig:fid_all}, Pseudo-LPIPS outperforms ${L}_1$ and squared ${L}_2$ by a substantial margin. Given the same value of $c$, Pseudo-LPIPS exhibits notably superior results compared to Pseudo-Huber metrics. Figure \ref{fig:fid_c} shows hyperparameter $c$ in our proposed metric plays a significant role in sample quality. When $c = 0 $, Pseudo-LPIPS degrades to LPIPS and it is clear that Pseudo-LPIPS consistently outperforms LPIPS even when the value of $c$ varies in a relatively wide range. These findings collectively validate the effectiveness and robustness of our proposed metrics. 
\begin{table}[]
    \centering
    \begin{tabular}{c|c|c}
         T&Epochs&FID(one-step)  \\
         \toprule
         1000&1000&2.90\\
         2000&2000&2.86\\
         4000&4000&2.83\\
         8000&8000&2.82\\
    \end{tabular}
    \caption{DDDM with different training configurations on CIFAR10.}
    \label{tab:higher}
\end{table}

\begin{figure*}[tb]
 \centering
\begin{subfigure}[t]{0.45\textwidth}
  \centering
  \includegraphics[scale=0.85]{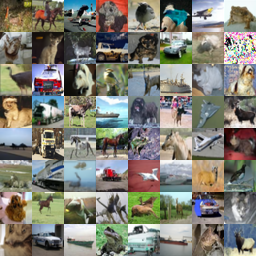}  
  \caption{one-step}
  \label{fig:cifar_1-step}
\end{subfigure}
\begin{subfigure}[t]{0.45\textwidth}
  \centering
  \includegraphics[scale=0.85]{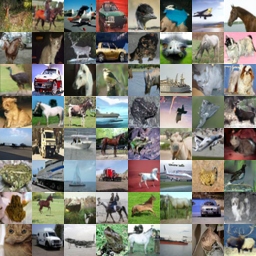}  
  \caption{two-step}
  \label{fig:ciar_2-step}
\end{subfigure}
\caption{One-step and two-step samples from DDDM-deep model trained on CIFAR-10}
\label{fig:Cifar}
\end{figure*}

\subsection{Comparison to SOTA} We compare our model against state-of-the-art generative models on CIFAR-10 and ImageNet 64x64. Quantitative results are summarized in Table \ref{tab:cifar} and Table \ref{tab:imagenet}. Our findings reveal that DDDMs exceed previous distillation diffusion models and methods that require advanced sampling procedures in both one-step and two-step generation on CIFAR-10 and ImageNet 64x64, which breaks the reliance on the well-pretrained diffusion models and simplifies the generation workflow. Moreover, our model demonstrates performance comparable to numerous leading generative models on both datasets. Specifically, baseline DDDM obtains FIDs of 2.90 and 2.79 for one-step and two-step generation on CIFAR-10, both results exceed that of StyleGAN2-ADA \cite{stylegan2-ada}. With deeper architecture and 1000-step sampling, DDDM-deep further reduces FID to 1.79, aligning with state-of-the-art method \cite{edmG}. It is worth noting that the leading few-step generation model iCT/iCT-deep \cite{iCT} is trained for 400k iterations, while our 
approach delivers competitive FIDs and higher IS scores under fewer training iterations. With $T=8000$ and trained for competitive epochs, DDDM achieves 2.82 and 1.74 on one-step and 1000-step generation respectively, both setting state-of-the-art performance. \\
On the ImageNet 64x64 dataset, DDDM attains FID scores of 4.21 and 3.53 for one-step and two-step generation, respectively. We have observed that iCT/iCT-deeper achieves superior results, benefitting from a 4x larger batch size and 1.6x more training iterations compared to our model. We hypothesize that the observed performance gap may be attributed to such computational resource disparities and suboptimal hyperparameters in our loss function. Despite these limitations, DDDM showcases improved precision and recall compared to iCT, demonstrating enhanced diversity and mode coverage while maintaining a similar model size.\\
The effectiveness of our iterative solution can also be clearly demonstrated by \cref{fig:fid_trend}. Overall, the FID consistently demonstrates a downward trend among different datasets and architectures, though it is not strictly monotonically decreasing with respect to the sampling steps.

\begin{figure}[h]
    \centering
    \includegraphics[scale=0.55]{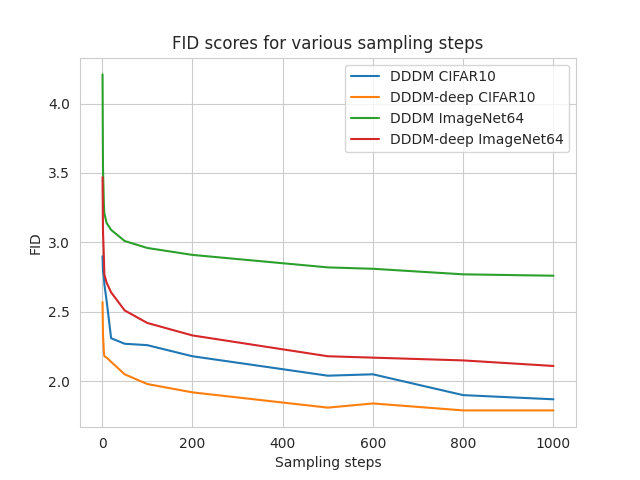}
    \caption{FID w.r.t inference iterations.}
    \label{fig:fid_trend}
\end{figure}

\section{Related Work}
The foundational work in diffusion probabilistic models (DPM) was initially conceptualized by \citeauthor{firstdiff} in 2015, where a generative Markov chain is developed to transfer the Gaussian distribution into the data distribution. Then \citet{ddpm} developed denoising diffusion probabilistic models (DDPM) and demonstrated their exceptional capabilities in image generation. By improving noise schedule and variance taking into consideration, \citeauthor{iddpm} further enhanced these models in 2021, achieving better log-likelihood scores and better FID scores. \citeauthor{score} focused on optimizing the score-matching objective and developed Noise Conditional Score Network (NCSN) \cite{iscore}. Despite their different motivations, DDPMs and NCSNs are closely related. Both DDPM and NCSN require many steps to achieve good sample quality and therefore have trouble generating high-quality samples in a few iterations. 

Many studies tried to reduce the sampling steps. DDIM \cite{ddim} has demonstrated effectiveness in few-step sampling, similar to the probability flow sampler. \citet{fastScore} examine fast stochastic differential equation integrators for reverse diffusion processes, and \citet{tzen2019theoretical} explore unbiased samplers conducive to fast, high-quality sampling. \cite{iddpm} and \citet{kong2021fast} describe methodologies for adapting discrete-time diffusion models. \citet{watson2021learning} have proposed a dynamic programming algorithm aimed at minimizing the number of timesteps required for a diffusion model, optimizing for log-likelihood.

Several studies have also shown the effectiveness of training diffusion models across continuous noise levels and subsequently tuning samplers post-training.  \citet{variational} involves adjusting the noise levels of a few-step discrete time reverse diffusion process. \citet{san2021noise} train a new network to estimate the noise level in noisy data, demonstrating how this estimate can expedite sampling. 

Modifications to the specifications of the diffusion model itself can also facilitate faster sampling. This includes altered forward and reverse processes, as studied by \citet{nachmani2021non} 
 and \citet{lam2021bilateral}. Moreover, Consistency models \cite{consistency} boost the sampling speed into a single step, in both consistency distillation (CD) and consistency training (CT). Later, by introducing several techniques, the one-step generation reaches state-of-the-art FID scores \cite{iCT}.

\section{Discussion and Limitations}
Since DDDM keeps track of $\mathbf{x}_0^{(n)}$ for each sample in the dataset, there will be additional memory consumption during training. Specifically, it requires extra 614MB for CIFAR10 and 29.5GB for ImageNet 64x64. Although it can be halved by using FP16 data type, such memory requirement might still be a challenge for larger dataset or dataset with high-resolution images.

We also notice that there could be bias in evaluation since the ImageNet is utilized in both LPIPS and Inception network for FID. The accidental leakage of ImageNet features from LPIPS may potentially lead to inflated FID scores. Other evaluation metrics such as human evaluations are needed to further validate our model. Furthermore, investigating unbiased loss for DDDM presents an interesting avenue for future research.

\section{Conclusion}
In conclusion, our presented DDDMs offer a straightforward and versatile approach to generating realistic images with minimal sampling steps and also support the iterative sampling process for better performance, eliminating the need for intricately designed samplers or distillation on pre-trained models. The core concept of our method involves conditioning the diffusion model on an estimated target generated from the previous training iteration, and during image generation, incorporating samples from previous timestep to guide the iterative process. Additionally, the incorporation of the proposed Pseudo-LPIPS enhances the robustness of our model, showcasing its potential for broader applications in other generative models. To further enhance the capabilities of DDDM and unleash its potential, we aim to leverage its strengths in a continuous-time setting.

\section*{Impact Statement }
This paper presents work whose goal is to advance the field of 
Machine Learning. There are many potential societal consequences 
of our work, none which we feel must be specifically highlighted here.

\section*{Acknowledgement}
Research primarily supported as part of the AIM for Composites, an Energy Frontier Research Center, an Energy Frontier Research Center funded by the U.S. Department of Energy (DOE), Office of Science, Basic Energy Sciences (BES), under Award \# DE-SC0023389 (Development of DDDM method) and by the US National Institute of Food and Agriculture (NIFA; Grant Number 2023-70029-41309) and the US National Science Foundation (NSF; Grant Number ABI-1759856, MTM2-2025541, OIA-2242812) (DDDM implementation and theoretical proof). In addition, the authors acknowledge research support from Clemson University with a generous allotment of computation time on the Palmetto cluster.





\nocite{langley00}

\bibliography{example_paper}

\begin{thebibliography}{54}
\providecommand{\natexlab}[1]{#1}
\providecommand{\url}[1]{\texttt{#1}}
\expandafter\ifx\csname urlstyle\endcsname\relax
  \providecommand{\doi}[1]{doi: #1}\else
  \providecommand{\doi}{doi: \begingroup \urlstyle{rm}\Url}\fi

\bibitem[Berthelot et~al.(2023)Berthelot, Autef, Lin, Yap, Zhai, Hu, Zheng, Talbot, and Gu]{tract}
Berthelot, D., Autef, A., Lin, J., Yap, D.~A., Zhai, S., Hu, S., Zheng, D., Talbot, W., and Gu, E.
\newblock Tract: Denoising diffusion models with transitive closure time-distillation.
\newblock \emph{arXiv preprint arXiv:2303.04248}, 2023.

\bibitem[Boos(1985)]{boos1985converse}
Boos, D.~D.
\newblock A converse to scheffe's theorem.
\newblock \emph{The Annals of Statistics}, pp.\  423--427, 1985.

\bibitem[Brock et~al.(2018)Brock, Donahue, and Simonyan]{bigGAN}
Brock, A., Donahue, J., and Simonyan, K.
\newblock Large scale gan training for high fidelity natural image synthesis.
\newblock \emph{arXiv preprint arXiv:1809.11096}, 2018.

\bibitem[Deng et~al.(2009)Deng, Dong, Socher, Li, Li, and Fei-Fei]{imagenet}
Deng, J., Dong, W., Socher, R., Li, L.-J., Li, K., and Fei-Fei, L.
\newblock Imagenet: A large-scale hierarchical image database.
\newblock In \emph{2009 IEEE conference on computer vision and pattern recognition}, pp.\  248--255. Ieee, 2009.

\bibitem[Dhariwal \& Nichol(2021)Dhariwal and Nichol]{ADM}
Dhariwal, P. and Nichol, A.
\newblock Diffusion models beat gans on image synthesis.
\newblock \emph{Advances in neural information processing systems}, 34:\penalty0 8780--8794, 2021.

\bibitem[Dinh et~al.(2016)Dinh, Sohl-Dickstein, and Bengio]{flow}
Dinh, L., Sohl-Dickstein, J., and Bengio, S.
\newblock Density estimation using real nvp.
\newblock \emph{arXiv preprint arXiv:1605.08803}, 2016.

\bibitem[Gu et~al.(2023)Gu, Zhai, Zhang, Liu, and Susskind]{boot}
Gu, J., Zhai, S., Zhang, Y., Liu, L., and Susskind, J.
\newblock Boot: Data-free distillation of denoising diffusion models with bootstrapping.
\newblock \emph{arXiv preprint arXiv:2306.05544}, 2023.

\bibitem[Heusel et~al.(2017)Heusel, Ramsauer, Unterthiner, Nessler, and Hochreiter]{FID}
Heusel, M., Ramsauer, H., Unterthiner, T., Nessler, B., and Hochreiter, S.
\newblock Gans trained by a two time-scale update rule converge to a local nash equilibrium.
\newblock \emph{Advances in neural information processing systems}, 30, 2017.

\bibitem[Ho et~al.(2020)Ho, Jain, and Abbeel]{ddpm}
Ho, J., Jain, A., and Abbeel, P.
\newblock Denoising diffusion probabilistic models.
\newblock \emph{Advances in neural information processing systems}, 33:\penalty0 6840--6851, 2020.

\bibitem[Ho et~al.(2022)Ho, Salimans, Gritsenko, Chan, Norouzi, and Fleet]{videodiff}
Ho, J., Salimans, T., Gritsenko, A.~A., Chan, W., Norouzi, M., and Fleet, D.~J.
\newblock Video diffusion models.
\newblock In \emph{ICLR Workshop on Deep Generative Models for Highly Structured Data}, 2022.
\newblock URL \url{https://openreview.net/forum?id=BBelR2NdDZ5}.

\bibitem[Jabri et~al.(2022)Jabri, Fleet, and Chen]{rin}
Jabri, A., Fleet, D., and Chen, T.
\newblock Scalable adaptive computation for iterative generation.
\newblock \emph{arXiv preprint arXiv:2212.11972}, 2022.

\bibitem[Jolicoeur-Martineau et~al.(2021)Jolicoeur-Martineau, Li, Pich{\'e}-Taillefer, Kachman, and Mitliagkas]{fastScore}
Jolicoeur-Martineau, A., Li, K., Pich{\'e}-Taillefer, R., Kachman, T., and Mitliagkas, I.
\newblock Gotta go fast when generating data with score-based models.
\newblock \emph{arXiv preprint arXiv:2105.14080}, 2021.

\bibitem[Karras et~al.(2020{\natexlab{a}})Karras, Laine, Aittala, Hellsten, Lehtinen, and Aila]{style2}
Karras, T., Laine, S., Aittala, M., Hellsten, J., Lehtinen, J., and Aila, T.
\newblock Analyzing and improving the image quality of stylegan.
\newblock In \emph{Proceedings of the IEEE/CVF conference on computer vision and pattern recognition}, pp.\  8110--8119, 2020{\natexlab{a}}.

\bibitem[Karras et~al.(2020{\natexlab{b}})Karras, Laine, Aittala, Hellsten, Lehtinen, and Aila]{stylegan2-ada}
Karras, T., Laine, S., Aittala, M., Hellsten, J., Lehtinen, J., and Aila, T.
\newblock Analyzing and improving the image quality of stylegan.
\newblock In \emph{Proceedings of the IEEE/CVF conference on computer vision and pattern recognition}, pp.\  8110--8119, 2020{\natexlab{b}}.

\bibitem[Karras et~al.(2022)Karras, Aittala, Aila, and Laine]{edm}
Karras, T., Aittala, M., Aila, T., and Laine, S.
\newblock Elucidating the design space of diffusion-based generative models.
\newblock \emph{Advances in Neural Information Processing Systems}, 35:\penalty0 26565--26577, 2022.

\bibitem[Kim et~al.(2022)Kim, Kim, Kwon, Kang, and Moon]{edmG}
Kim, D., Kim, Y., Kwon, S.~J., Kang, W., and Moon, I.-C.
\newblock Refining generative process with discriminator guidance in score-based diffusion models.
\newblock \emph{arXiv preprint arXiv:2211.17091}, 2022.

\bibitem[Kim et~al.(2023)Kim, Lai, Liao, Murata, Takida, Uesaka, He, Mitsufuji, and Ermon]{kim2023consistency}
Kim, D., Lai, C.-H., Liao, W.-H., Murata, N., Takida, Y., Uesaka, T., He, Y., Mitsufuji, Y., and Ermon, S.
\newblock Consistency trajectory models: Learning probability flow ode trajectory of diffusion.
\newblock \emph{arXiv preprint arXiv:2310.02279}, 2023.

\bibitem[Kingma et~al.(2021)Kingma, Salimans, Poole, and Ho]{variational}
Kingma, D., Salimans, T., Poole, B., and Ho, J.
\newblock Variational diffusion models.
\newblock \emph{Advances in neural information processing systems}, 34:\penalty0 21696--21707, 2021.

\bibitem[Kingma \& Dhariwal(2018)Kingma and Dhariwal]{glow}
Kingma, D.~P. and Dhariwal, P.
\newblock Glow: Generative flow with invertible 1x1 convolutions.
\newblock \emph{Advances in neural information processing systems}, 31, 2018.

\bibitem[Kong \& Ping(2021)Kong and Ping]{kong2021fast}
Kong, Z. and Ping, W.
\newblock On fast sampling of diffusion probabilistic models.
\newblock \emph{arXiv preprint arXiv:2106.00132}, 2021.

\bibitem[Krizhevsky et~al.(2009)Krizhevsky, Hinton, et~al.]{CIFAR}
Krizhevsky, A., Hinton, G., et~al.
\newblock Learning multiple layers of features from tiny images.
\newblock 2009.

\bibitem[Kynk{\"a}{\"a}nniemi et~al.(2019)Kynk{\"a}{\"a}nniemi, Karras, Laine, Lehtinen, and Aila]{Precision}
Kynk{\"a}{\"a}nniemi, T., Karras, T., Laine, S., Lehtinen, J., and Aila, T.
\newblock Improved precision and recall metric for assessing generative models.
\newblock \emph{Advances in Neural Information Processing Systems}, 32, 2019.

\bibitem[Lam et~al.(2021)Lam, Wang, Huang, Su, and Yu]{lam2021bilateral}
Lam, M.~W., Wang, J., Huang, R., Su, D., and Yu, D.
\newblock Bilateral denoising diffusion models.
\newblock \emph{arXiv preprint arXiv:2108.11514}, 2021.

\bibitem[Langley(2000)]{langley00}
Langley, P.
\newblock Crafting papers on machine learning.
\newblock In Langley, P. (ed.), \emph{Proceedings of the 17th International Conference on Machine Learning (ICML 2000)}, pp.\  1207--1216, Stanford, CA, 2000. Morgan Kaufmann.

\bibitem[Liu et~al.(2022)Liu, Gong, and Liu]{rectified}
Liu, X., Gong, C., and Liu, Q.
\newblock Flow straight and fast: Learning to generate and transfer data with rectified flow.
\newblock \emph{arXiv preprint arXiv:2209.03003}, 2022.

\bibitem[Lu et~al.(2022)Lu, Zhou, Bao, Chen, Li, and Zhu]{dpm-solver}
Lu, C., Zhou, Y., Bao, F., Chen, J., Li, C., and Zhu, J.
\newblock Dpm-solver: A fast ode solver for diffusion probabilistic model sampling in around 10 steps.
\newblock \emph{Advances in Neural Information Processing Systems}, 35:\penalty0 5775--5787, 2022.

\bibitem[Luhman \& Luhman(2021)Luhman and Luhman]{knowledge}
Luhman, E. and Luhman, T.
\newblock Knowledge distillation in iterative generative models for improved sampling speed.
\newblock \emph{arXiv preprint arXiv:2101.02388}, 2021.

\bibitem[Luo et~al.(2023)Luo, Hu, Zhang, Sun, Li, and Zhang]{instruct}
Luo, W., Hu, T., Zhang, S., Sun, J., Li, Z., and Zhang, Z.
\newblock Diff-instruct: A universal approach for transferring knowledge from pre-trained diffusion models.
\newblock \emph{arXiv preprint arXiv:2305.18455}, 2023.

\bibitem[Nachmani et~al.(2021)Nachmani, Roman, and Wolf]{nachmani2021non}
Nachmani, E., Roman, R.~S., and Wolf, L.
\newblock Non gaussian denoising diffusion models.
\newblock \emph{arXiv preprint arXiv:2106.07582}, 2021.

\bibitem[Nichol et~al.(2021)Nichol, Dhariwal, Ramesh, Shyam, Mishkin, McGrew, Sutskever, and Chen]{glide}
Nichol, A., Dhariwal, P., Ramesh, A., Shyam, P., Mishkin, P., McGrew, B., Sutskever, I., and Chen, M.
\newblock Glide: Towards photorealistic image generation and editing with text-guided diffusion models.
\newblock \emph{arXiv preprint arXiv:2112.10741}, 2021.

\bibitem[Nichol \& Dhariwal(2021)Nichol and Dhariwal]{iddpm}
Nichol, A.~Q. and Dhariwal, P.
\newblock Improved denoising diffusion probabilistic models.
\newblock In \emph{International Conference on Machine Learning}, pp.\  8162--8171. PMLR, 2021.

\bibitem[Ramesh et~al.(2022)Ramesh, Dhariwal, Nichol, Chu, and Chen]{unclip}
Ramesh, A., Dhariwal, P., Nichol, A., Chu, C., and Chen, M.
\newblock Hierarchical text-conditional image generation with clip latents, 2022.
\newblock \emph{URL https://arxiv. org/abs/2204.06125}, 7, 2022.

\bibitem[Rombach et~al.(2022)Rombach, Blattmann, Lorenz, Esser, and Ommer]{latentDiff}
Rombach, R., Blattmann, A., Lorenz, D., Esser, P., and Ommer, B.
\newblock High-resolution image synthesis with latent diffusion models.
\newblock In \emph{Proceedings of the IEEE/CVF conference on computer vision and pattern recognition}, pp.\  10684--10695, 2022.

\bibitem[Saharia et~al.(2022{\natexlab{a}})Saharia, Chan, Saxena, Li, Whang, Denton, Ghasemipour, Gontijo~Lopes, Karagol~Ayan, Salimans, et~al.]{Imagen}
Saharia, C., Chan, W., Saxena, S., Li, L., Whang, J., Denton, E.~L., Ghasemipour, K., Gontijo~Lopes, R., Karagol~Ayan, B., Salimans, T., et~al.
\newblock Photorealistic text-to-image diffusion models with deep language understanding.
\newblock \emph{Advances in Neural Information Processing Systems}, 35:\penalty0 36479--36494, 2022{\natexlab{a}}.

\bibitem[Saharia et~al.(2022{\natexlab{b}})Saharia, Ho, Chan, Salimans, Fleet, and Norouzi]{sr3}
Saharia, C., Ho, J., Chan, W., Salimans, T., Fleet, D.~J., and Norouzi, M.
\newblock Image super-resolution via iterative refinement.
\newblock \emph{IEEE Transactions on Pattern Analysis and Machine Intelligence}, 45\penalty0 (4):\penalty0 4713--4726, 2022{\natexlab{b}}.

\bibitem[Salimans et~al.(2016)Salimans, Goodfellow, Zaremba, Cheung, Radford, and Chen]{IS}
Salimans, T., Goodfellow, I., Zaremba, W., Cheung, V., Radford, A., and Chen, X.
\newblock Improved techniques for training gans.
\newblock \emph{Advances in neural information processing systems}, 29, 2016.

\bibitem[San-Roman et~al.(2021)San-Roman, Nachmani, and Wolf]{san2021noise}
San-Roman, R., Nachmani, E., and Wolf, L.
\newblock Noise estimation for generative diffusion models.
\newblock \emph{arXiv preprint arXiv:2104.02600}, 2021.

\bibitem[Sohl-Dickstein et~al.(2015)Sohl-Dickstein, Weiss, Maheswaranathan, and Ganguli]{firstdiff}
Sohl-Dickstein, J., Weiss, E., Maheswaranathan, N., and Ganguli, S.
\newblock Deep unsupervised learning using nonequilibrium thermodynamics.
\newblock In \emph{International conference on machine learning}, pp.\  2256--2265. PMLR, 2015.

\bibitem[Song et~al.(2020{\natexlab{a}})Song, Meng, and Ermon]{ddim}
Song, J., Meng, C., and Ermon, S.
\newblock Denoising diffusion implicit models.
\newblock \emph{arXiv preprint arXiv:2010.02502}, 2020{\natexlab{a}}.

\bibitem[Song \& Dhariwal(2023)Song and Dhariwal]{iCT}
Song, Y. and Dhariwal, P.
\newblock Improved techniques for training consistency models.
\newblock \emph{arXiv preprint arXiv:2310.14189}, 2023.

\bibitem[Song \& Ermon(2020)Song and Ermon]{iscore}
Song, Y. and Ermon, S.
\newblock Improved techniques for training score-based generative models.
\newblock \emph{Advances in neural information processing systems}, 33:\penalty0 12438--12448, 2020.

\bibitem[Song et~al.(2020{\natexlab{b}})Song, Sohl-Dickstein, Kingma, Kumar, Ermon, and Poole]{score}
Song, Y., Sohl-Dickstein, J., Kingma, D.~P., Kumar, A., Ermon, S., and Poole, B.
\newblock Score-based generative modeling through stochastic differential equations.
\newblock \emph{arXiv preprint arXiv:2011.13456}, 2020{\natexlab{b}}.

\bibitem[Song et~al.(2023)Song, Dhariwal, Chen, and Sutskever]{consistency}
Song, Y., Dhariwal, P., Chen, M., and Sutskever, I.
\newblock Consistency models.
\newblock \emph{arXiv preprint arXiv:2303.01469}, 2023.

\bibitem[Srivastava et~al.(2014)Srivastava, Hinton, Krizhevsky, Sutskever, and Salakhutdinov]{dropout}
Srivastava, N., Hinton, G., Krizhevsky, A., Sutskever, I., and Salakhutdinov, R.
\newblock Dropout: a simple way to prevent neural networks from overfitting.
\newblock \emph{The journal of machine learning research}, 15\penalty0 (1):\penalty0 1929--1958, 2014.

\bibitem[Sweeting(1986)]{sweeting1986converse}
Sweeting, T.
\newblock On a converse to scheff{\'e}'s theorem.
\newblock \emph{The Annals of Statistics}, 14\penalty0 (3):\penalty0 1252--1256, 1986.

\bibitem[Tzen \& Raginsky(2019)Tzen and Raginsky]{tzen2019theoretical}
Tzen, B. and Raginsky, M.
\newblock Theoretical guarantees for sampling and inference in generative models with latent diffusions.
\newblock In \emph{Conference on Learning Theory}, pp.\  3084--3114. PMLR, 2019.

\bibitem[Vahdat \& Kautz(2020)Vahdat and Kautz]{nvae}
Vahdat, A. and Kautz, J.
\newblock Nvae: A deep hierarchical variational autoencoder.
\newblock \emph{Advances in neural information processing systems}, 33:\penalty0 19667--19679, 2020.

\bibitem[Vahdat et~al.(2021)Vahdat, Kreis, and Kautz]{LSGM}
Vahdat, A., Kreis, K., and Kautz, J.
\newblock Score-based generative modeling in latent space.
\newblock \emph{Advances in Neural Information Processing Systems}, 34:\penalty0 11287--11302, 2021.

\bibitem[Watson et~al.(2021)Watson, Ho, Norouzi, and Chan]{watson2021learning}
Watson, D., Ho, J., Norouzi, M., and Chan, W.
\newblock Learning to efficiently sample from diffusion probabilistic models.
\newblock \emph{arXiv preprint arXiv:2106.03802}, 2021.

\bibitem[Xu et~al.(2022)Xu, Liu, Tegmark, and Jaakkola]{pfgm}
Xu, Y., Liu, Z., Tegmark, M., and Jaakkola, T.
\newblock Poisson flow generative models.
\newblock \emph{Advances in Neural Information Processing Systems}, 35:\penalty0 16782--16795, 2022.

\bibitem[Zhang \& Chen(2022)Zhang and Chen]{3deis}
Zhang, Q. and Chen, Y.
\newblock Fast sampling of diffusion models with exponential integrator.
\newblock \emph{arXiv preprint arXiv:2204.13902}, 2022.

\bibitem[Zhang et~al.(2018)Zhang, Isola, Efros, Shechtman, and Wang]{lpips}
Zhang, R., Isola, P., Efros, A.~A., Shechtman, E., and Wang, O.
\newblock The unreasonable effectiveness of deep features as a perceptual metric.
\newblock In \emph{Proceedings of the IEEE conference on computer vision and pattern recognition}, pp.\  586--595, 2018.

\bibitem[Zhao et~al.(2023)Zhao, Bai, Rao, Zhou, and Lu]{unipc}
Zhao, W., Bai, L., Rao, Y., Zhou, J., and Lu, J.
\newblock Unipc: A unified predictor-corrector framework for fast sampling of diffusion models.
\newblock \emph{arXiv preprint arXiv:2302.04867}, 2023.

\bibitem[Zheng et~al.(2023)Zheng, Nie, Vahdat, Azizzadenesheli, and Anandkumar]{DFNO}
Zheng, H., Nie, W., Vahdat, A., Azizzadenesheli, K., and Anandkumar, A.
\newblock Fast sampling of diffusion models via operator learning.
\newblock In \emph{International Conference on Machine Learning}, pp.\  42390--42402. PMLR, 2023.

\end{thebibliography}
\bibliographystyle{icml2024}

\newpage
\appendix
\onecolumn
\section{Appendix}
\subsection{Derivation of the definition of $\mathbf{f}\left(\mathbf{x}_0, \mathbf{x}_t, t\right)$}
\label{appen:1}
Starting with Eq. \eqref{equ:ode} we find that the integration of $\mathbf{x}_t$ from time $t$ to $0$ is given by:

$$\int_t^0 \frac{\mathrm{d} \mathbf{x}_s}{\mathrm{~d} s} \mathrm{~d} s =\int_t^0-\frac{1}{2} \beta(s)\left[\mathbf{x}_s-\nabla_{x_s} \log q_s\left(\mathbf{x}_s\right)\right] \mathrm{d} s
$$

Thus,
$$
    \mathbf{x}_0 - \mathbf{x}_t = - \int_t^0\frac{1}{2} \beta(s)\left[\mathbf{x}_s-\nabla_{x_s} \log q_s\left(\mathbf{x}_s\right)\right] \mathrm{d} s
$$

Identifying the right-hand side of this equation as a function of $\left(\mathbf{x}_0, \mathbf{x}_t, t\right)$ allows us to introduce the $\mathbf{F}(\mathbf{x}_0, \mathbf{x}_t, t)$. Consequently, it can be reformulated as:

$$\mathbf{x}_0-\mathbf{x}_t=-\mathbf{F}\left(\mathbf{x}_0, \mathbf{x}_t, t\right) \Longrightarrow \mathbf{x}_0=\mathbf{x}_t-\mathbf{F}\left(\mathbf{x}_0, \mathbf{x}_t, t\right),$$

leading to the definition:
$$\mathbf{f}\left(\mathbf{x}_0, \mathbf{x}_t, t\right)=\mathbf{x}_t-\mathbf{F}\left(\mathbf{x}_0, \mathbf{x}_t, t\right).$$

In our case, $\mathbf{f}\left(\mathbf{x}_0, \mathbf{x}_t, t\right) =  \mathbf{x}_0\left(\mathbf{x}_t, t\right)$ suggests that $\mathbf{x}_0$ is a function of $\mathbf{x}_t$ and $t$, but it is embedded within a larger function $f$ that equates to $\mathbf{x}_0$. This setup implies an implicit relationship between $\mathbf{x}_t, t$, and $\mathbf{x}_0$.
In our implicit case, no direct expression is present, and we often cannot isolate one of the variables on one side of the equation without involving the others. Thus, it lets the neural network function estimate all the unstable parts. 

In contrast, in DDPM, $\mathbf{x}_0$ is approximated as $\hat{\mathbf{x}}_0=\left(\mathbf{x}_t-\sqrt{1-\bar{\alpha}_t} \epsilon_\theta\left(\mathbf{x}_t, t\right)\right) / \sqrt{\bar{\alpha}_t}$, presenting a partially explicit framework for relating $(\mathbf{x}_t, t)$ to $\mathbf{x}_0$. This equation, though it provides a method to estimate $\mathbf{x}_0$, highlights the potential for numerical instability. The division by $\sqrt{\bar{\alpha}_t}$ can amplify errors in estimating the noise $\epsilon_\theta\left(\mathbf{x}_t, t\right)$, especially as $\bar{\alpha}_t$ becomes small, which is typical in the latter stages of the reverse process where the data is more significantly noised.


\newpage
\section{Additional Samples}
In this section, we provide additional samples from our models.


\begin{figure}[h]

  \centering
  \includegraphics[scale=0.58]{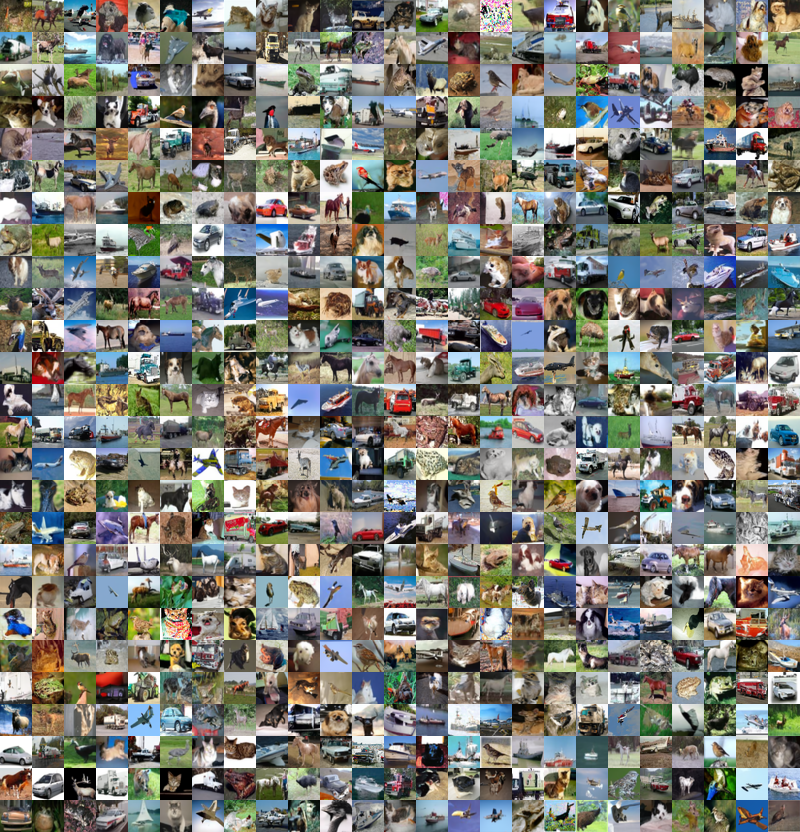}  
  \caption{One-step samples from DDDM model trained on CIFAR-10}
  \label{fig:1}

\end{figure}
\begin{figure}[h]

  \centering
  \includegraphics[scale=0.58]{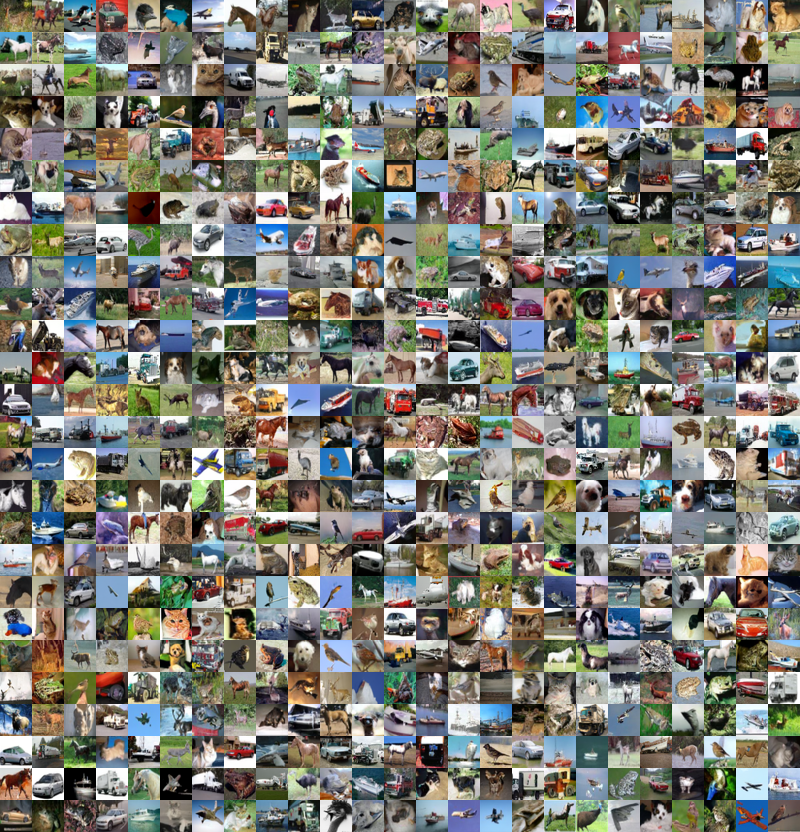}  
  \caption{Two-step samples from DDDM model trained on CIFAR-10}
  \label{fig:2}

\end{figure}
\begin{figure}[h]

  \centering
  \includegraphics[scale=0.58]{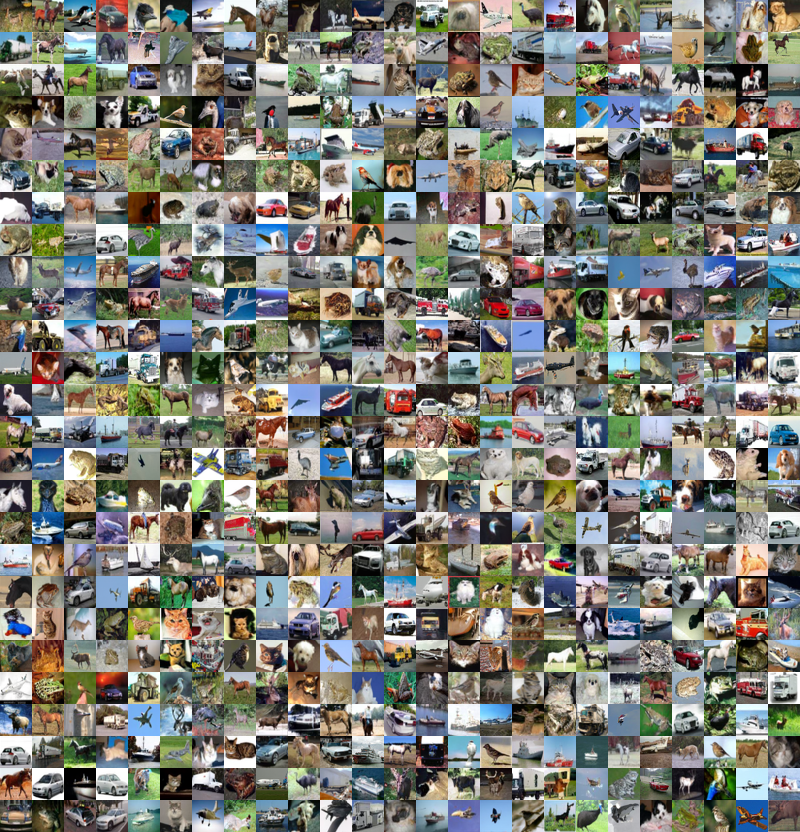}  
  \caption{1000-step samples from DDDM model trained on CIFAR-10}
  \label{fig:3}

\end{figure}
\begin{figure}[h]

  \centering
  \includegraphics[scale=0.58]{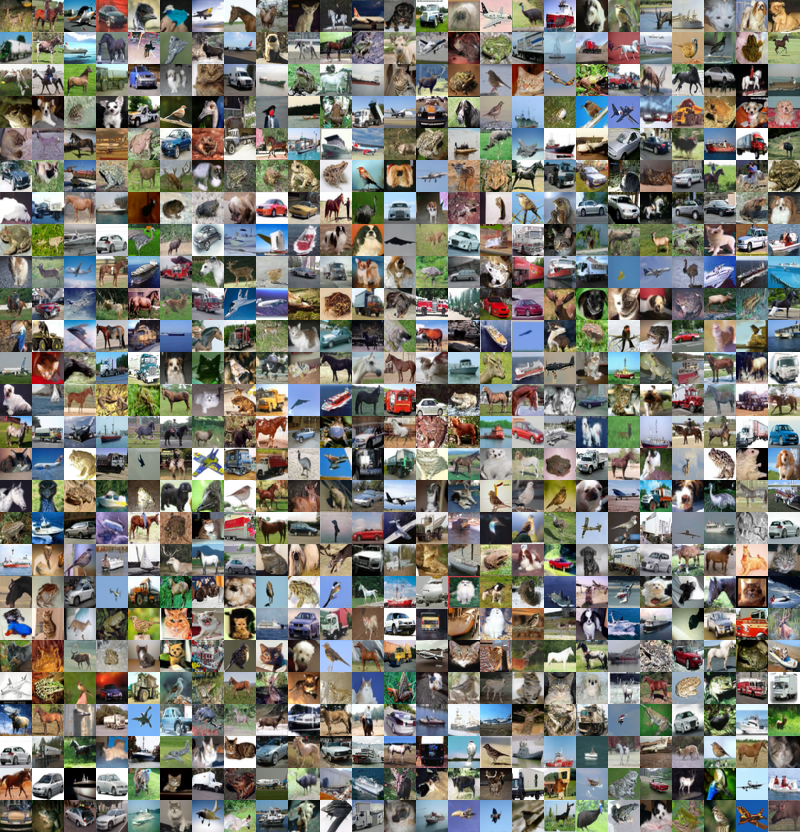}  
  \caption{One-step samples from DDDM-deep model trained on CIFAR-10}
  \label{fig:4}

\end{figure}
\begin{figure}[h]

  \centering
  \includegraphics[scale=0.58]{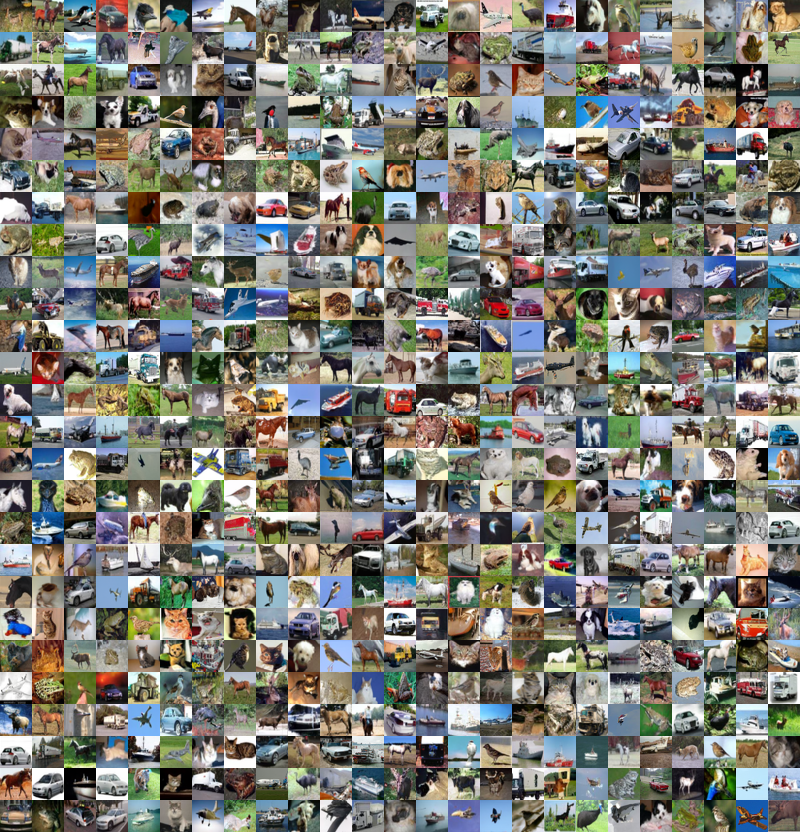}  
  \caption{Two-step samples from DDDM-deep model trained on CIFAR-10}
  \label{fig:5}

\end{figure}
\begin{figure}[h]

  \centering
  \includegraphics[scale=0.58]{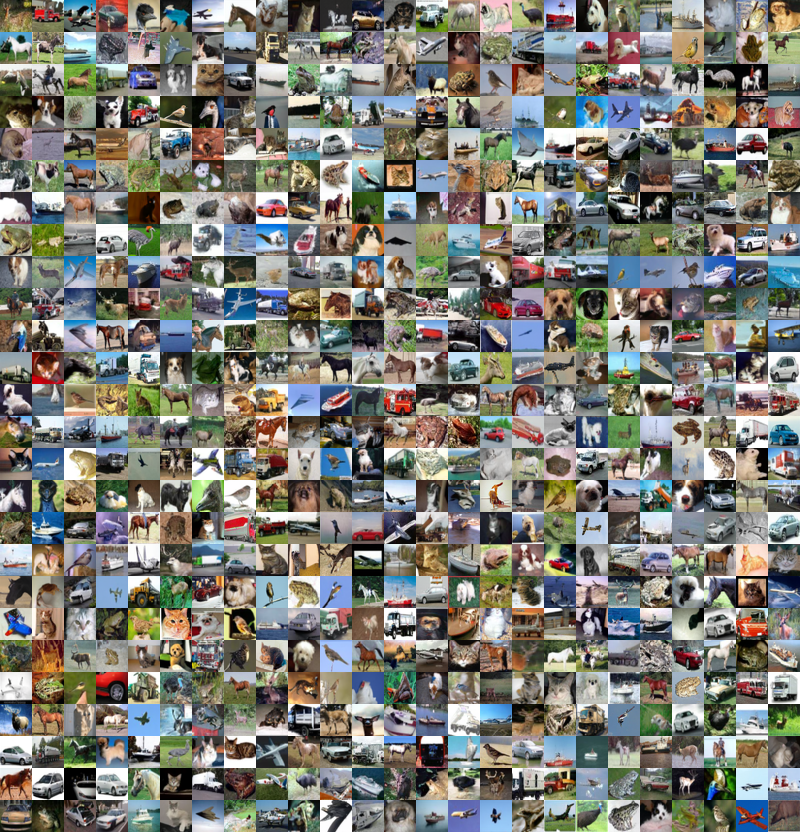}  
  \caption{1000-step samples from DDDM-deep model trained on CIFAR-10}
  \label{fig:6}

\end{figure}

\begin{figure}[h]

  \centering
  \includegraphics[scale=0.65]{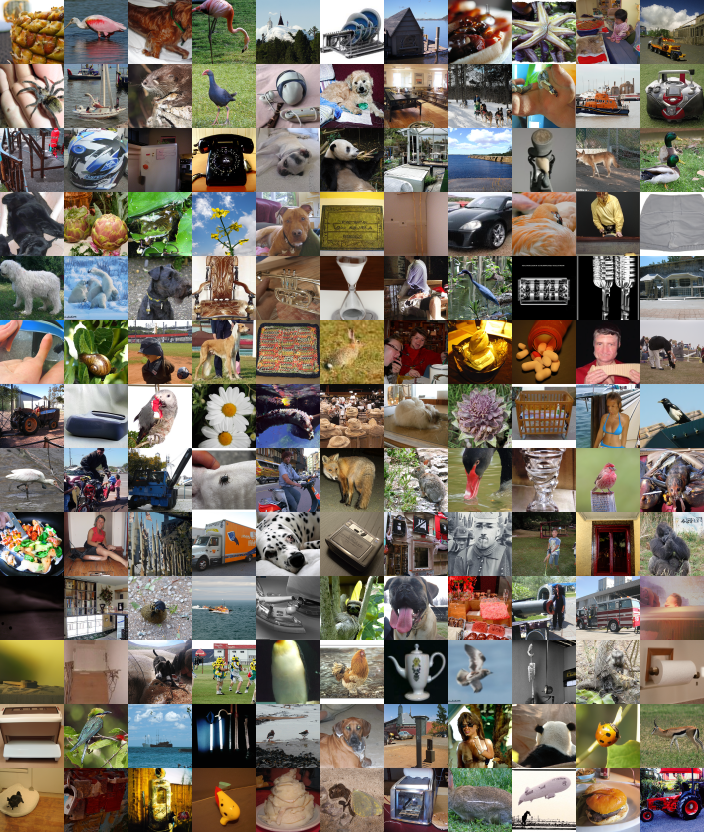}  
  \caption{one-step samples from DDDM model trained on ImageNet 64x64}
  \label{fig:7}

\end{figure}

\begin{figure}[h]

  \centering
  \includegraphics[scale=0.65]{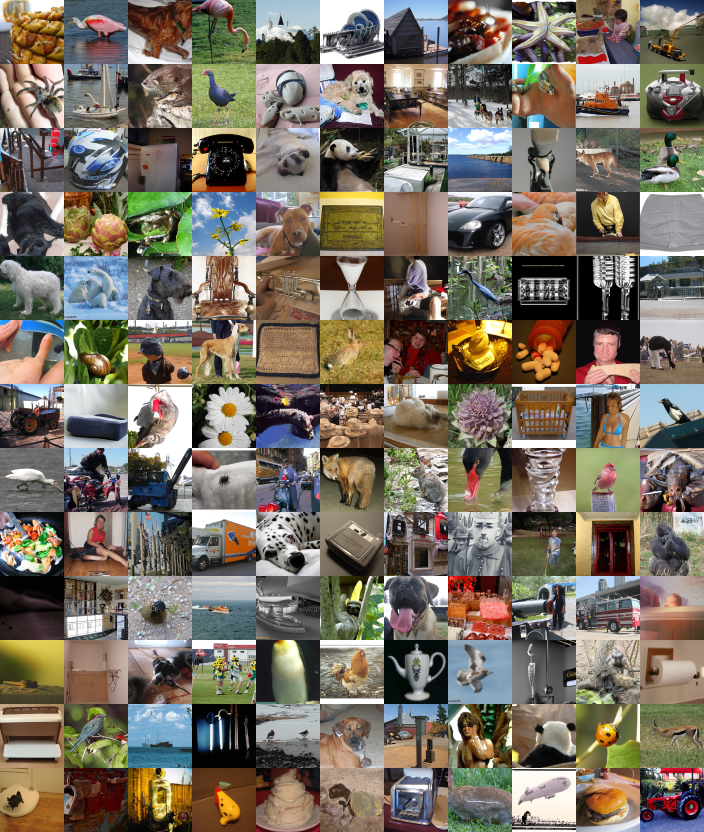}  
  \caption{two-step samples from DDDM model trained on ImageNet 64x64}
  \label{fig:8}

\end{figure}

\begin{figure}[h]

  \centering
  \includegraphics[scale=0.65]{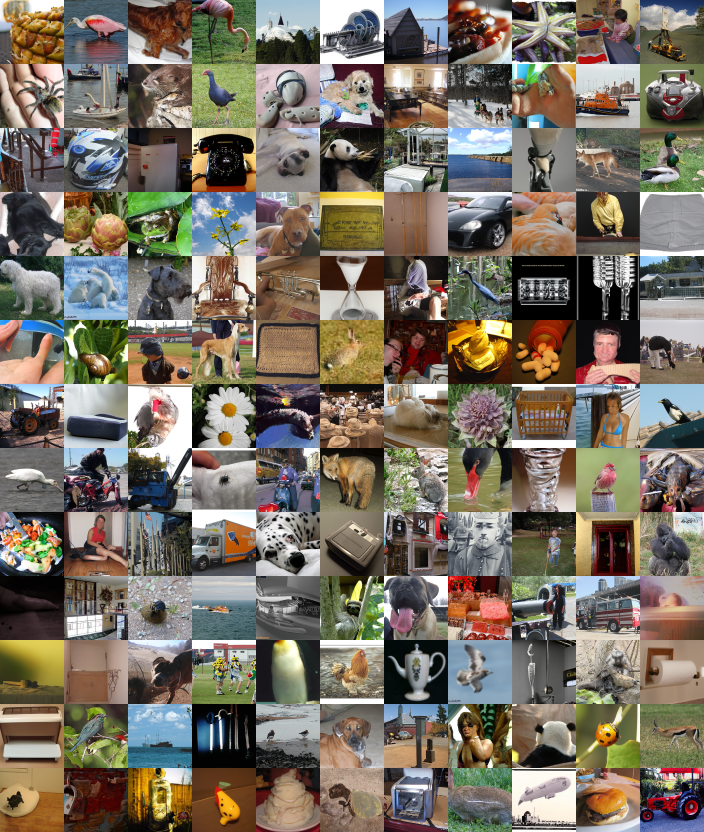}  
  \caption{1000-step samples from DDDM model trained on ImageNet 64x64}
  \label{fig:9}

\end{figure}

\begin{figure}[h]

  \centering
  \includegraphics[scale=0.65]{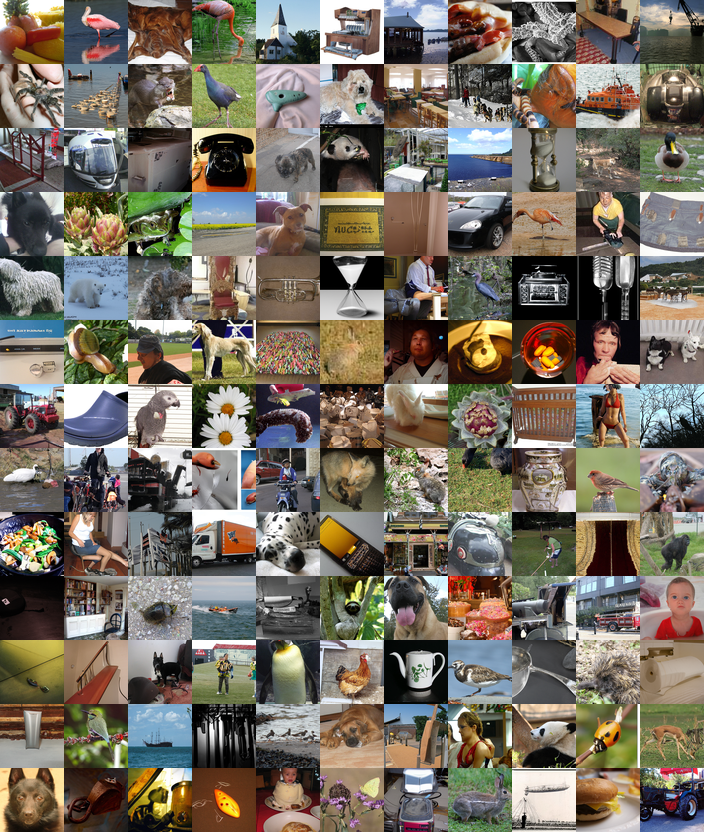}  
  \caption{one-step samples from DDDM-deep model trained on ImageNet 64x64}
  \label{fig:10}

\end{figure}

\begin{figure}[h]

  \centering
  \includegraphics[scale=0.65]{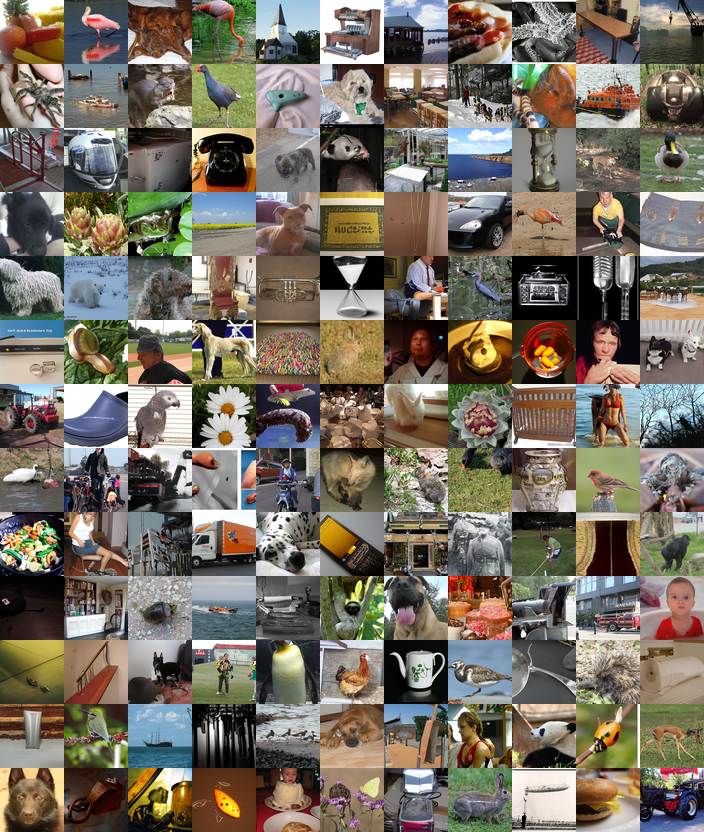}  
  \caption{two-step samples from DDDM-deep model trained on ImageNet 64x64}
  \label{fig:11}

\end{figure}

\begin{figure}[h]

  \centering
  \includegraphics[scale=0.65]{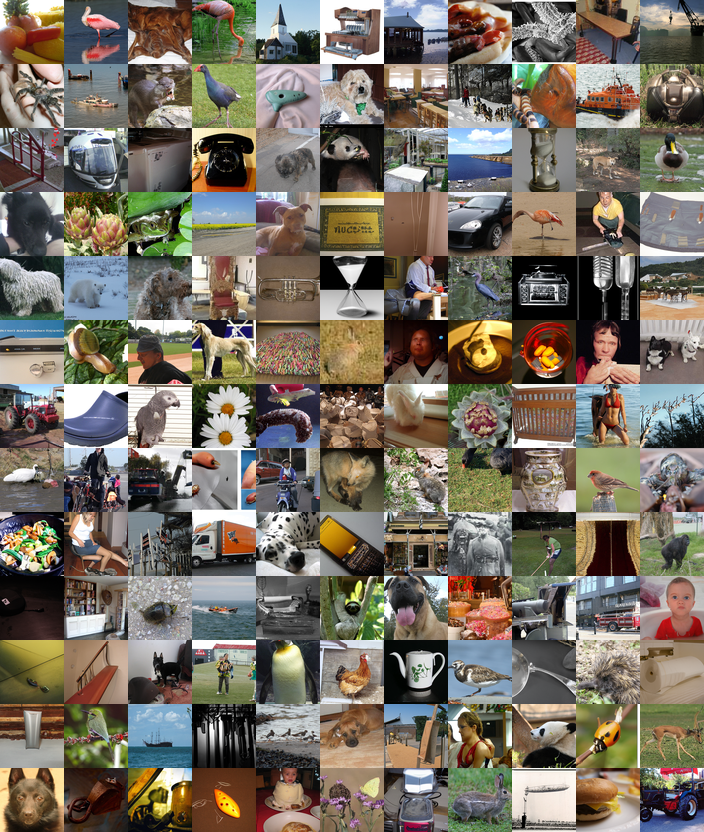}  
  \caption{1000-step samples from DDDM-deep model trained on ImageNet 64x64}
  \label{fig:12}

\end{figure}

\end{document}